%% file: main.tex
\documentclass[runningheads, envcountsame, a4paper]{llncs}

\usepackage{amsmath}
\usepackage{amssymb}

\DeclareMathOperator*{\argmin}{arg\,min}

\newcommand{\prob}{\mathbb{P}}
\newcommand{\expect}{\mathbb{E}}
\usepackage{graphicx}
\usepackage{hyperref}
\usepackage{url}

\begin{document}

\title{Shift Happens: Adjusting Classifiers}%\thanks{Supported by organization x.}}

\author{Theodore James Thibault Heiser\orcidID{0000-0001-7057-3160} \and
Mari-Liis Allikivi\orcidID{0000-0002-1019-3454} \and
Meelis Kull\orcidID{0000-0001-9257-595X}}
\authorrunning{T. Heiser et al.}

\institute{Institute of Computer Science, University of Tartu, Tartu, Estonia\\
\email{\{mari-liis.allikivi,meelis.kull\}@ut.ee}}

\maketitle

\begin{abstract}
Minimizing expected loss measured by a proper scoring rule, such as Brier score or log-loss (cross-entropy), is a common objective while training a probabilistic classifier. If the data have experienced dataset shift where the class distributions change post-training, then often the model's performance will decrease, over-estimating the probabilities of some classes while under-estimating the others on average. We propose unbounded and bounded general adjustment (UGA and BGA) methods that transform all predictions to (re-)equalize the average prediction and the class distribution. These methods act differently depending on which proper scoring rule is to be minimized, and we have a theoretical guarantee of reducing loss on test data, if the exact class distribution is known. We also demonstrate experimentally that, when in practice the class distribution is known only approximately, there is often still a reduction in loss depending on the amount of shift and the precision to which the class distribution is known.

\end{abstract}

%\begin{enumerate}
%    \item Make sure we are careful with the terms 'expected class distribution' vs 'class distribution' and 'expected loss' vs 'loss'. Possibly elaborate this in the text also.
%    \item Make sure we fulfill the promise from the intro: 'Section 2 of this paper provides the background for this work, covering the specific types of dataset shift and reviewing some popular methods of adapting to them.'
%    \item Make sure we analyse experimentally the impact of different kinds of shift, as in background we write: 'While the proposed adjustment methods are in principle applicable for any kind of dataset shift, there are differences in performance across different types of shift.'
%\end{enumerate}

\input{sections/1_introduction.tex}

\input{sections/2_background.tex}

\input{sections/3_theory.tex}

\input{sections/4_experiments.tex}

\input{sections/5_conclusion.tex}
\bibliographystyle{splncs04}
\bibliography{references}

\input{supplement}

\end{document}

%% file: sections/1_introduction.tex
\section{Introduction}

Classical supervised machine learning is built on the assumption that the joint probability distribution that features and labels are sourced from does not change during the life cycle of the predictive model: from training to testing and deployment. However, in reality this assumption is broken more often than not: medical diagnostic classifiers are often trained with an oversampling of disease-positive instances, surveyors are often biased to collecting labelled samples from certain segments of a population, user demographics and preferences change over time on social media and e-commerce sites, etc.

While these are all examples of dataset shift, the nature of these shifts can be quite different.
%Dataset shift can have many different causes and the resulting shift can also manifest itself quite differently. 
There have been several efforts to create taxonomies of dataset shift \cite{storkey2009training,moreno2012unifying}. The field of \emph{transfer learning} offers many methods of learning models for scenarios with awareness of the shift during training. However, often the shift is not yet known during training and it is either too expensive or even impossible to retrain once the shift happens. There are several reasons for it: original training data or training infrastructure might not be available; shift happens so frequently that there is no time to retrain; the kind of shift is such that without having labels in the shifted context there is no hope of learning a better model than the original.%\par 

In this work we address multi-class classification scenarios where training a classifier for the shifted deployment context is not possible (due to any of the above reasons), and the only possibility is to post-process the outputs from an existing classifier that was trained before the shift happened.
To succeed, such post-processing must be guided by some information about the shifted deployment context. 
In the following, we will assume that we know the overall expected class distribution in the shifted context, at least approximately. 
For example, consider a medical diagnostic classifier of disease sub-types, which has been trained on the cases of country A, and gets deployed to a different country B. It is common that the distribution of sub-types can vary between countries, but in many cases such information is available. So here many labels are available but not the feature values (country B has data about sub-types in past cases, but no diagnostic markers were measured back then), 
%. in the deployment context but no feature information is available, , as the past cases in country B  (disease sub-types of past cases in country B)  (the diagnostic markers not measured on the past cases of country B), 
making training of a new model impossible. Still, the model \emph{adjustment} methods proposed in this paper can be used to adjust the existing model to match the class distribution in the deployment context.
As another example, consider a bank's fraud detection classifier trained on one type of credit cards and deployed to a new type of credit cards. For new cards there might not yet be enough cases of fraud to train a new classifier, but there might be enough data to estimate the class distribution, that is the prevalence of fraud. The old classifier might predict too few or too many positives on the new data, so it must be adjusted to the new class distribution. %\par

In many application domains, including the above examples of medical diagnostics and fraud detection, it is required that the classifiers would output confidence information in addition to the predicted class. This is supported by most classifiers, as they can be requested to provide for each instance the class probabilities instead of a single label. For example, the feed-forward neural networks for classification typically produce class probabilities using the final soft-max layer. Such confidence information can then be interpreted by a human expert to choose the action based on the prediction, or feeded into an automatic cost-sensitive decision-making system, which would use the class probability estimates and the mis-classification cost information to make cost-optimal decisions. Probabilistic classifiers are typically evaluated using \emph{Brier score} or \emph{log-loss} (also known as \emph{squared error} and \emph{cross-entropy}, respectively). Both measures belong to the family of proper scoring rules: measures which are minimized by the true posterior class probabilities produced by the Bayes-optimal model. Proper losses also encourage the model to produce calibrated probabilities, as every proper loss decomposes into \emph{calibration loss} and \emph{refinement loss} \cite{kull2015novel}.

Our goal is to improve the predictions of a given model in a shifted deployment context, using the information about the expected class distribution in this context, without making any additional assumptions about the type of dataset shift. The idea proposed by Kull et al. \cite{kull2015novel} is to take advantage of a property that many dataset shift cases share: a difference in the classifier's average prediction and the expected class distribution of the data. They proposed two different \emph{adjustment procedures} which transform the predictions to re-equalise the average prediction with the expected class distribution, resulting in a theoretically guaranteed reduction of Brier score or log-loss. Interestingly, it turned out that different loss measures require different adjustment procedures. They proved that their proposed \emph{additive adjustment} (additively shifting all predictions, see Section~\ref{sec:background} for the definitions) is guaranteed to reduce Brier score, while it can increase log-loss in some circumstances. They also proposed \emph{multiplicative adjustment} (multiplicatively shifting and renormalising all predictions) which is guaranteed to reduce log-loss, while it can sometimes increase Brier score. It was proved that if the adjustment procedure is \emph{coherent} with the proper loss (see Section~\ref{sec:background}), then the reduction of loss is guaranteed, assuming that the class distribution is known exactly. They introduced the term \emph{adjustment loss} to refer to the part of calibration loss which can be eliminated by adjustment. Hence, adjustment can be viewed as a weak form of calibration.
In the end, it remained open: (1) whether for every proper scoring rule there exists an adjustment procedure that is guaranteed to reduce loss; (2) is there a general way of finding an adjustment procedure to reduce a given proper loss; (3) whether this reduction of loss from adjustment materializes in practice where the new class distribution is only known approximately; (4) how to solve algorithm convergence issues of the multiplicative adjustment method; (5) how to solve the problem of additive adjustment sometimes producing predictions with negative 'probabilities'.% \par

The contributions of our work are the following: (1) we construct a family called BGA (Bounded General Adjustment) of adjustment procedures, with one procedure for each proper loss, and prove that each BGA procedure is guaranteed to reduce the respective proper loss, if the class distribution of the dataset is known; (2) we show that each BGA procedure can be represented as a convex optimization task, leading to a practical and tractable algorithm; (3) we demonstrate experimentally that even if the new class distribution is only known approximately, the proposed BGA methods still usually improve over the unadjusted model; (4) we prove that the BGA procedure of log-loss is the same as multiplicative adjustment, thus solving the convergence problems of multiplicative adjustment; (5) we construct another family called UGA (Unbounded General Adjustment) with adjustment procedures that are dominated by the respective BGA methods according to the loss, but are theoretically interesting by being coherent to the respective proper loss in the sense of Kull et al. \cite{kull2015novel}, and by containing the additive adjustment procedure as the UGA for Brier score.%\par 
 
%The contributions of our work are the following: (1) we prove that two adjustment procedures, unbounded and bounded general adjustment (abbreviated as UGA and BGA), can be applied for any proper scoring rule and reduce expected loss on a dataset if the class distribution of the dataset is known; (2) we show that both UGA and BGA can be represented as convex optimization tasks, leading to practical and tractable algorithms; (3) we demonstrate experimentally that even if the new class distribution is only known approximately, the proposed adjustment methods still usually improve over the unadjusted model.\par

Section 2 of this paper provides the background for this work, covering the specific types of dataset shift and reviewing some popular methods of adapting to them. We also review the family of proper losses, i.e. the loss functions that adjustment is designed for. Section 3 introduces the UGA and BGA families of adjustment procedures and provides the theoretical results of the paper. Section 4 provides experimental evidence for the effectiveness of BGA adjustment in practical settings. 
%We used an open database \cite{OpenML2013} to sample a large amount of datasets and sets of predictions from different models, and then measured the performance of adjustment after inducing various types of shift. 
Section 5 concludes the paper, reviewing its contributions and proposing open questions.

%% file: sections/2_background.tex
\section{Background and Related Work}\label{sec:background}

%\begin{enumerate}
%    \item I'm not really sure of the usefulness of presenting the unbounded version of the transformation, since the projection there happens into a set that does not only include probabilities and is dominated by the bounded version. The main struggle with the paper is how the adjustment that decomposes has that problem of mapping outside probabilities, while the adjustment that maps to actual probabilities do not decompose. I'd appreciate an ellaboration around this point, since without it the paper is quite confusing. Similarly, further ellaboration on the benefits of coherent adjustements would be key to understand the value of the paper.
%    \item Define "alpha"/adjustment procedure formally in section 2.2, define that it is less than calibration
%    \item clarify class distribution is a vector
%    \item clarify that we are discussing 1-hot representations
%    \item Expand upon the previous works (transfer learning)
%\end{enumerate}

\subsection{Dataset Shift and Prior Probability Adjustment}

In supervised learning, dataset shift can be defined as any change in the joint probability distribution of the feature vector $X$ and label $Y$ between two data generating processes, that is $\prob_{old}(X,Y) \neq \prob_{new}(X,Y)$, where $\prob_{old}$ and $\prob_{new}$ are the probability distributions before and after the shift, respectively.
While the proposed adjustment methods are in principle applicable for any kind of dataset shift, there are differences in performance across different types of shift. 
According to Moreno-Torres et al \cite{moreno2012unifying} there are 4 main kinds of shift: covariate shift, prior probability shift, concept shift and other types of shift.
\emph{Covariate shift} is when the distribution $\prob(X)$ of the covariates/features changes, but the posterior class probabilities $P(Y\vert X)$ do not. At first, this may not seem to be of much interest since the classifiers output estimates of posterior class probabilities and these remain unshifted. However, unless the classifier is Bayes-optimal, then covariate shift can still result in a classifier under-performing \cite{storkey2009training}. Many cases of covariate shift can be modelled as sample selection bias \cite{hein2009binary}, often addressed by retraining the model on a reweighted training set \cite{sugiyama2007covariate,shimodaira2000improving,gretton2009covariate}.
%In other cases, the covariates take values outside the range of the training set.
%, and making it hard to reuse the orequiring additional assumptions about the to .\par
%
%\paragraph{Prior Probability Shift}
%$$\textit{Prior Probability Shift:}\qquad\prob_{old}(Y) \neq \prob_{new}(Y), \quad \prob_{old}(X \vert Y) = \prob_{new}(X \vert Y)$$
\emph{Prior probability shift} is when the prior class probabilities $\prob(Y)$ change, but the likelihoods $\prob(X\vert Y)$ do not. An example of this is down- or up-sampling of the instances based on their class in the training or testing phase. Given the new class distribution, the posterior class probability predictions can be modified according to Bayes' theorem to take into account the new prior class probabilities, as shown in \cite{saerens2002adjusting}. We will refer to this procedure as the \emph{Prior Probability Adjuster} (PPA) and the formal definition is as follows:
%\begin{definition} [Prior Probability Adjuster (PPA)]
$$\text{PPA:}\qquad\mathbb{P}_{new} (Y{=}\,y \vert X) = \frac{\mathbb{P}_{old} (Y{=}\,y \vert X) \prob_{new}(Y{=}\,y) / \prob_{old}(Y{=}\,y)}{\sum_{y'} \mathbb{P}_{old} (Y{=}\,y' \vert X) \prob_{new}(Y{=}\,y') / \prob_{old}(Y{=}\,y') } $$
%\end{definition}
In \emph{other types of shift} both conditional probability distributions $\prob(X\vert Y)$ and $\prob(Y\vert X)$ change.
The special case where $\prob(Y)$ or $\prob(X)$ remains unchanged is called \emph{concept shift}. Concept shift and other types of shift are in general hard to adapt to, as the relationship between $X$ and $Y$ has changed in an unknown way.

\subsection{Proper Scoring Rules and Bregman Divergences}

The best possible probabilistic classifier is the Bayes-optimal classifier which for any instance $X$ outputs its true posterior class probabilities $\prob(Y\vert X)$. 
%For probabilistic classifiers, the intuitive understanding is that the classifier should ideally, given an instance, ideally output the true posterior probability for each class. 
When choosing a loss function for evaluating probabilistic classifiers, it is then natural to require that the loss would be minimized when the predictions match the correct posterior probabilities. Loss functions with this property are called proper scoring rules \cite{dawid2007geometry,merkle2013choosing,kull2015novel}.
Note that throughout the paper we consider multi-class classification with $k$ classes and represent class labels as one-hot vectors, i.e. the label of class $i$ is a vector of $k-1$ zeros and a single $1$ at position $i$. 
%
%\vspace{-0.3pt}
\begin{definition}[Proper Scoring Rule (or Proper Loss)]
In a $k$-class classification task a loss function $f:[0,1]^k\times\{0,1\}^k\to\mathbb{R}$ is called a \emph{proper scoring rule} (or \emph{proper loss}), if for any probability vectors $p,q\in[0,1]^k$ with $\sum_{i=1}p_i=1$ and $\sum_{i=1}q_i=1$ the following inequality holds:
%\begin{align*}
$$\expect_{Y \sim q}[f(q,Y)] \leq \expect_{Y \sim q}[f(p,Y)]$$ %\\[-0.3cm]
%\end{align*}
where $Y$ is a one-hot encoded label randomly drawn from the categorical distribution over $k$ classes with class probabilities represented by vector $q$. 
%Given a loss function $f$, a prediction of the probability distribution of labels $p$, a random variable $Y$ representing a label, and the true probability distribution of labels $q$, we define $f$ to be a proper scoring rule if
The loss function $f$ is called \emph{strictly proper} if the inequality is strict for all $p\neq q$.
\end{definition}
%
%\vspace{-0.3pt}
This is a useful definition, but it does not give a very clear idea of what the geometry of these functions looks like. Bregman divergences \cite{bregman1967relaxation} were developed independently of proper scoring rules and have a constructive definition (note that many authors have the arguments $p$ and $q$ the other way around, but we use this order to match proper losses).
%
%\vspace{-0.3pt}
\begin{definition}[Bregman Divergence]
Let $\phi : \Omega \to \mathbb{R}$ be a strictly convex function defined on a convex set $\Omega \subseteq \mathbb{R}^k$ such that $\phi$ is differentiable on the relative interior of $\Omega$, $ri(\Omega)$. The Bregman divergence $d_\phi : ri(\Omega) \times \Omega \to [0, \infty )$ is defined as 
%\vspace{-0.3pt}
%\begin{align*}
$$d_\phi (p,q) = \phi (q) - \phi(p) - \langle q - p , \nabla \phi (p) \rangle$$ %\\[-0.5cm]
%\end{align*}
\end{definition}
Previous works \cite{banerjee2005optimality} have shown that the two concepts are closely related. Every Bregman divergence is a strictly proper scoring rule and every strictly proper scoring rule (within an additive constant) is a Bregman divergence. Best known functions in these families are %Popular functions in both families include squared Euclidean distance and KL divergence.
\emph{squared Euclidean distance} defined as 
$d_{SED}(\mathbf{p}, \mathbf{q}) = \sum_{j=1}^d (p_j - q_j)^2$
and \emph{Kullback-Leibler-divergence} 
$d_{KL}(\mathbf{p}, \mathbf{q}) = \sum_{j=1}^{d} q_j \log \frac{q_j}{p_j}$.
%\vspace{-0.5cm}
%\begin{align*}
%\text{Squared Euclidean Distance:\:} 
%& d_{SED}(\mathbf{p}, \mathbf{q}) = {\lVert \mathbf{q} - \mathbf{p} \rVert}^2 = \sum_{j=1}^d (q_j - p_j)^2 \\[-0.4cm]
%\text{KL Divergence:\:} 
%& d_{KL}(\mathbf{p}, \mathbf{q}) = \sum_{j=1}^{d} q_j \log \frac{q_j}{p_j} \\[-0.4cm]
%\end{align*}
When used as a scoring rule to measure loss of a prediction against labels, they are typically referred to as \emph{Brier Score} $d_{BS}$, and \emph{log-loss} $d_{LL}$, respectively.
%Hence, we will also use the alternative notation $d_{BS}=d_{SED}$ and $d_{LL}=d_{KL}$.% \par

\subsection{Adjusted Predictions and Adjustment Procedures}

Let us now come to the main scenario of this work, where dataset shift of unknown type occurs after a probabilistic $k$-class classifier has been trained. Suppose that we have a test dataset with $n$ instances from the post-shift distribution. We denote the predictions of the existing probabilistic classifier on these data by $p\in[0,1]^{n\times k}$, where $p_{ij}$ is the predicted class $j$ probability on the $i$-th instance, and hence $\sum_{j=1}^k p_{ij}=1$. We further denote the hidden actual labels in the one-hot encoded form by $y\in\{0,1\}^{n\times k}$, where $y_{ij}=1$ if the $i$-th instance belongs to class $j$, and otherwise $y_{ij}=0$.
While the actual labels are hidden, we assume that the overall class distribution $\pi\in[0,1]^k$ is known, where $\pi_j=\frac{1}{n}\sum_{i=1}^n y_{ij}$. The following theoretical results require $\pi$ to be known exactly, but in the experiments we demonstrate benefits from the proposed adjustment methods also in the case where $\pi$ is known approximately. As discussed in the introduction, examples of such scenarios include medical diagnostics and fraud detection. Before introducing the adjustment procedures we define what we mean by \emph{adjusted predictions}.

\begin{definition}[Adjusted Predictions]
Let $p\in[0,1]^{n\times k}$ be the predictions of a probabilistic $k$-class classifier on $n$ instances and let $\pi\in[0,1]^k$ be the actual class distribution on these instances. We say that \emph{predictions $p$ are adjusted on this dataset}, if the average prediction is equal to the class proportion for every class $j$, that is
$\frac{1}{n}\sum_{i=1}^n p_{ij}=\pi_j$.
\end{definition}

Essentially, the model provides adjusted predictions on a dataset, if for each class its predicted probabilities on the given data are on average neither under- nor over-estimated. Note that this definition was presented in \cite{kull2015novel} using random variables and expected values, and our definition can be viewed as a finite case where a random instance is drawn from the given dataset.

Consider now the case where the predictions are not adjusted on the given test dataset, and so the estimated class probabilities are on average too large for some class(es) and too small for some other class(es). This raises a question of whether the overall loss (as measured with some proper loss) could be reduced by shifting all predictions by a bit, for example with additive shifting by adding the same constant vector $\varepsilon$ to each prediction vector $p_{i\cdot}$. The answer is not obvious as in this process some predictions would also be moved further away from their true class. This is in some sense analogous to the case where a regression model is on average over- or under-estimating its target, as there also for some instances the predictions would become worse after shifting. However, additive shifting still pays off, if the regression results are evaluated by mean squared error. This is well known from the theory of linear models where mean squared error fitting leads to an intercept value such that the average predicted target value on the training set is equal to the actual mean target value (unless regularisation is applied). Since Brier score is essentially the same as mean squared error, it is natural to expect reduction of Brier score after additive shifting of predictions towards the actual class distribution. This is indeed so, and \cite{kull2015novel} proved that \emph{additive adjustment} guarantees a reduction of Brier score. Additive adjustment is a method which adds the same constant vector to all prediction vectors to achieve equality between average prediction vector and actual class distribution.

\begin{definition}[Additive Adjustment]
\emph{Additive adjustment} is the function $\alpha_{+}:[0,1]^{n\times k}\times [0,1]^k\to [0,1]^{n\times k}$ which takes in the predictions of a probabilistic $k$-class classifier on $n$ instances and the actual class distribution $\pi$ on these instances, and outputs adjusted predictions $a=\alpha_{+}(p,\pi)$ defined as
%\begin{align*}
$a_{i\cdot}=p_{i\cdot}+(\varepsilon_1,\dots,\varepsilon_k)$
%\end{align*}
where $a_{i\cdot}=(a_{i1},\dots,a_{ik})$, $p_{i\cdot}=(p_{i1},\dots,p_{ik})$, and $\varepsilon_j=\pi_j-\frac{1}{n}\sum_{i=1}^n p_{ij}$ for each class $j\in\{1,\dots,k\}$.
\end{definition}

It is easy to see that additive adjustment procedure indeed results in adjusted predictions, as $\frac{1}{n}\sum_{i=1}^n a_{ij}=\frac{1}{n}\sum_{i=1}^n p_{ij}+\varepsilon_j=\pi_j$. 
Note that even if the original predictions $p$ are probabilities between $0$ and $1$, the additively adjusted predictions $a$ can sometimes go out from that range and be negative or larger than $1$. For example, if an instance $i$ is predicted to have probability $p_{ij}=0$ to be in class $j$ and at the same time on average the overall proportion of class $j$ is over-estimated, then $\varepsilon_j<0$ and the adjusted prediction $a_{ij}=\varepsilon_j$ is negative. While such predictions are no longer probabilities in the standard sense, these can still be evaluated with Brier score. So it is always true that the overall Brier score on adjusted predictions is lower than on the original predictions, $\frac{1}{n}d_{BS}(a_{i\cdot},y_{i\cdot})\leq \frac{1}{n}d_{BS}(p_{i\cdot},y_{i\cdot})$, 
where the equality holds only when the original predictions are already adjusted, that is $a=p$.
%\todo{refer to the theorem}
Note that whenever we mention the guaranteed reduction of loss, it always means (even if we do not re-emphasise it) that there is no reduction in the special case where the predictions are already adjusted, since then adjustment has no effect.

Additive adjustment is just one possible transformation of unadjusted predictions into adjusted predictions, and there are infinitely many other such transformations. 
We will refer to these as \emph{adjustment procedures}. If we have explicitly required the output values to be in the range $[0,1]$ then we use the term \emph{bounded adjustment procedure}, otherwise we use the term \emph{unbounded adjustment procedure}, even if actually the values do not go out from that range.

\begin{definition}[Adjustment Procedure]
\emph{Adjustment procedure} is any function $\alpha:[0,1]^{n\times k}\times [0,1]^k\to [0,1]^{n\times k}$ which takes as arguments the predictions $p$ of a probabilistic $k$-class classifier on $n$ instances and the actual class distribution $\pi$ on these instances, such that for any $p$ and $\pi$ the output predictions $a=\alpha(p,\pi)$ are adjusted, that is $\frac{1}{n}\sum_{i=1}^n a_{ij}=\pi_j$ for each class $j\in\{1,\dots,k\}$.
\end{definition}

In this definition and also in the rest of the paper we assume silently, that $p$ contains valid predictions of a probabilistic classifier, and so for each instance $i$ the predicted class probabilities add up to $1$, that is $\sum_{j=1}^k p_{ij}=1$. Similarly, we assume that $\pi$ contains a valid class distribution, with $\sum_{j=1}^k\pi_j=1$.

\begin{definition}[Bounded Adjustment Procedure]
An adjustment procedure $\alpha:[0,1]^{n\times k}\times [0,1]^k\to [0,1]^{n\times k}$ is \emph{bounded}, if for any $p$ and $\pi$ the output predictions $a=\alpha(p,\pi)$ are in the range $[0,1]$, that is $a_{ij}\in[0,1]$ for all $i,j$.
\end{definition}

An example of a bounded adjustment procedure is the \emph{multiplicative adjustment} method proposed in \cite{kull2015novel}, which multiplies the prediction vector component-wise with a constant weight vector and renormalizes the result to add up to 1.

\begin{definition}[Multiplicative Adjustment]
\emph{Multiplicative adjustment} is the function $\alpha_{*}:[0,1]^{n\times k}\times [0,1]^k\to [0,1]^{n\times k}$ which takes in the predictions of a probabilistic $k$-class classifier on $n$ instances and the actual class distribution $\pi$ on these instances, and outputs adjusted predictions $a=\alpha_{*}(p,\pi)$ defined as 
%\begin{align*}
$a_{ij}=\frac{w_j p_{ij}}{z_i}$,
%\end{align*}
where $w_1,\dots,w_k\geq 0$ are real-valued weights chosen based on $p$ and $\pi$ such that the predictions $\alpha_{*}(p,\pi)$ would be adjusted, and $z_i$ are the renormalisation factors defined as $z_i=\sum_{j=1}^k w_j p_{ij}$. 
\end{definition}

As proved in \cite{kull2015novel}, the suitable class weights $w_1,\dots,w_k$ are guaranteed to exist, but finding these weights is a non-trivial task and the algorithm based on coordinate descent proposed in \cite{kull2015novel} can sometimes fail to converge. 
In the next Section~\ref{sec:theory} we will propose a different and more reliable algorithm for multiplicative adjustment.

%One might wonder what the purpose of having different adjustment procedures is. 
It turns out that the adjustment procedure should be selected depending on which proper scoring rule is aimed to be minimised. It was proved in \cite{kull2015novel} that Brier score is guaranteed to be reduced with additive adjustment and log-loss with multiplicative adjustment. It was shown that when the 'wrong' adjustment method is used, then the loss can actually increase. In particular, additive adjustment can increase log-loss and multiplicative adjustment can increase Brier score. A sufficient condition for a guaranteed reduction of loss is \emph{coherence} between the adjustment procedure and the proper loss corresponding to a Bregman divergence.
Intuitively, coherence means that the effect of adjustment is the same across instances, where the effect is measured as the difference of divergences of this instance from any fixed class labels $j$ and $j'$. The exact definition is the following:

\begin{definition}[Coherence of Adjustment Procedure and Bregman Divergence \cite{kull2015novel}]
Let $\alpha:[0,1]^{n\times k}\times [0,1]^k\to [0,1]^{n\times k}$ be an adjustment procedure and $d_\phi$ be a Bregman divergence. 
Then $\alpha$ is called to be coherent with $d_\phi$ if and only if for any predictions $p$ and class distribution $\pi$ the following holds for all $i=1,\dots,n$ and $j,j'=1,\dots,k$:
$$\left(d_\phi(a_{i\cdot},c_j)-d_\phi(p_{i\cdot},c_j)\right)
 -\left(d_\phi(a_{i\cdot},c_{j'})-d_\phi(p_{i\cdot},c_{j'})\right)
 =const_{j,j'}$$
where $const_{j,j'}$ is a quantity not depending on $i$, and where $a=\alpha(p,\pi)$ and $c_j$ is a one-hot encoded vector corresponding to class $j$ (with $1$ at position $j$ and $0$ everywhere else).
\end{definition}

The following result can be seen as a direct corollary of Theorem~4 in \cite{kull2015novel}.

\begin{theorem}[Decomposition of Bregman Divergences \cite{kull2015novel}]\label{thm:decomp}
Let $d_\phi$ be a Bregman divergence and let $\alpha:[0,1]^{n\times k}\times [0,1]^k\to [0,1]^{n\times k}$ be an adjustment procedure coherent with $d_\phi$.
Then for any predictions $p$, one-hot encoded true labels $y\in\{0,1\}^{n\times k}$ and class distribution $\pi$ (with $\pi_j=\frac{1}{n}\sum_{i=1}^n y_{ij}$) the following decomposition holds:
\begin{align} \label{eq:decomp}
\frac{1}{n}\sum_{i=1}^n d_\phi(p_{i\cdot},y_{i\cdot}) =
\frac{1}{n}\sum_{i=1}^n d_\phi(p_{i\cdot},a_{i\cdot}) +
\frac{1}{n}\sum_{i=1}^n d_\phi(a_{i\cdot},y_{i\cdot})
\end{align}
\end{theorem}

Due to non-negativity of $d_\phi$ this theorem gives a guaranteed reduction of loss, that is the loss on the adjusted probabilities $a$ (average divergence between $a$ and $y$) is less than the loss on the original unadjusted probabilities (average divergence between $p$ and $y$), unless the probabilities are already adjusted ($p=a$). As additive adjustment can be shown to be coherent with the squared Euclidean distance and multiplicative adjustment with KL-divergence \cite{kull2015novel}, the respective guarantees of loss reduction follow from Theorem~\ref{thm:decomp}.% \par

However, from \cite{kull2015novel} it remained open: (1) whether for every proper loss there exists an adjustment procedure that is guaranteed to reduce loss; (2) is there a general way of finding an adjustment procedure to reduce a given proper loss; (3) whether this reduction of loss from adjustment materializes in practice where the new class distribution is only known approximately; (4) how to solve algorithm convergence issues of the multiplicative adjustment method; (5) how to solve the problem of additive adjustment sometimes producing predictions with negative 'probabilities'.
We will answer these questions in the next section.

%% file: sections/3_theory.tex
\section{General Adjustment} \label{sec:theory}

%\begin{enumerate}
%    \item I'm not really sure of the usefulness of presenting the unbounded version of the transformation, since the projection there happens into a set that does not only include probabilities and is dominated by the bounded version. The main struggle with the paper is how the adjustment that decomposes has that problem of mapping outside probabilities, while the adjustment that maps to actual probabilities do not decompose. I'd appreciate an ellaboration around this point, since without it the paper is quite confusing. Similarly, further ellaboration on the benefits of coherent adjustements would be key to understand the value of the paper.
%\end{enumerate}

The contributions of our work are the following: (1) we construct a family called BGA (Bounded General Adjustment) of adjustment procedures, with one procedure for each proper loss, and prove that each BGA procedure is guaranteed to reduce the respective proper loss, if the class distribution of the dataset is known; (2) we show that each BGA procedure can be represented as a convex optimization task, leading to a practical and tractable algorithm; (3) we demonstrate experimentally that even if the new class distribution is only known approximately, the proposed BGA methods still usually improve over the unadjusted model; (4) we prove that the BGA procedure of log-loss is the same as multiplicative adjustment, thus solving the convergence problems of multiplicative adjustment; (5) we construct another family called UGA (Unbounded General Adjustment) with adjustment procedures that are dominated by the respective BGA methods according to the loss, but are theoretically interesting by being coherent to the respective proper loss in the sense of Kull et al. \cite{kull2015novel}, and by containing the additive adjustment procedure as the UGA for Brier score.%\par 

Our main contribution is a family of adjustment procedures called BGA (Bounded General Adjustment). We use the term 'general' to emphasise that it is not a single method, but a family with exactly one adjustment procedure for each proper loss. We will prove that every adjustment procedure of this family is guaranteed to reduce the respective proper loss, assuming that the true class distribution is known exactly. To obtain more theoretical insights and answer the open questions regarding coherence of adjustment procedures with Bregman divergences and proper losses, we define a weaker variant of BGA called UGA (Unbounded General Adjustment). As the name says, these methods can sometimes output predictions that are not in the range $[0,1]$. On the other hand, the UGA procedures turn out to be coherent with their corresponding divergence measure, and hence have the decomposition stated in Theorem~\ref{thm:decomp} and also guarantee reduced loss. However, UGA procedures have less practical value, as each UGA procedure is dominated by the respective BGA in terms of reductions in loss. We start by defining the UGA procedures, as these are mathematically simpler.

%, we first present unbounded general adjustment (UGA) which is a general form of all coherent adjusters. As a special case, it coincides with additive and multiplicative adjustment for Brier score and log-loss, respectively, and provides a convex optimisation task, leading to a practical adjustment algorithm. However, like its special additive adjustment case, it can produce scores outside the [0,1] bounds for many proper scoring rules. To solve that problem, we introduce bounded general adjustment (BGA). BGA keeps scores in the [0,1] bounds and guarantees even better expected loss reduction than UGA, however it is not coherent and does not offer a decomposition. In pratical applications of adjustment, it is the clear choice. 

\subsection{Unbounded General Adjustment (UGA)}

We work here with the same notations as introduced earlier, with $p$ denoting the $n\times k$ matrix with the outputs of a $k$-class probabilistic classifier on a test dataset with $n$ instances, and $y$ denoting the matrix with the same shape containing one-hot encoded actual labels. 
%Each instance $i \in [1 \ldots n]$ has a matching row $p_i$, where each entry $p_{i,j}$ is a real number indicating the probability of that instance being class $j \in [1 \ldots k]$. \par %A scoring classifier does not necessarily need to output meaningful posterior probabilities for each class, SVMs for example do not.\par
We denote the unknown true posterior class probabilities $\prob(Y\vert X)$ on these instances by $q$, again a matrix with the same shape as $p$ and $y$. 
%If our classifier were Bayes-optimal, then we would have $p=q$. Due to the defining property of any proper loss $d_\phi$, the average loss $\frac{1}{n}\sum_{i=1}^n d_\phi(p_{i\cdot},y_{i\cdot})$ of predictions $p$ against labels $y$ is in expectation (over all $y$ drawn from distributions $q$) minimised at $p=q$.
%However, we are concerned with minimizing the actual loss on the given test dataset with unknown labels $y$, rather than the expected loss. For this we assume that 

Our goal is to reduce the loss $\frac{1}{n}\sum_{i=1}^n d_\phi(p_{i\cdot},y_{i\cdot})$ knowing the overall class distribution $\pi$, while not having any other information about labels $y$. Due to the defining property of any proper loss, the expected value of this quantity is minimised at $p=q$. As we know neither $y$ nor $q$, we consider instead the set of all possible predictions $Q_\pi$ that are adjusted to $\pi$, that is
%\begin{align} \label{eq:qpi}
$Q_\pi=\left\{a\in \mathbb{R}^{n \times k}\ \middle\vert\ \frac{1}{n}\sum_{i=1}^n a_{i,j} = \pi_j,\ \sum_{j=1}^k a_{i,j} = 1\right\}$.
%\end{align} 
Note that here we do not require $a_{ij}\geq 0$, as in this subsection we are working to derive unbounded adjustment methods which allow predictions to go out from the range $[0,1]$. 
%instead set the assume that , but unfortunately we  average loss In expectation, this loss would be minimised by The goal of the classifier then becomes to output $q$, the set of true posterior probabilities, minimizing our expected loss. At first, without any additional information given, the set of possible matrices $q$ is completely unknown, so we can represent the possible values as the set $Q = \mathbb{R}^{n \times k}$. We can first restrict Q with the equality constraint that all rows in any $a\in Q$ should sum to one: $\sum_{j=1}^k a_{i,j} = 1 \:\forall\: i  \in [1, n]$. The reader might think at this point that we should also now add an inequality constraint to each entry so that (combined with the previous constraint) each cell in all possible $q$ matrices is between 0 and 1, another requirement of being a probability. Let's for now ignore this requirement. We will add this when we introduce BGA.\par % to create a practical adjuster.\par

The question is now whether there exists a prediction matrix $a\in Q_\pi$ that is better than $p$ (i.e. has lower divergence from $y$) regardless of what the actual labels $y$ are (as a sidenote, $y$ also belongs to $Q_\pi$). %  \in Q^\star_\pi$ are. 
It is not obvious that such $a$ exists, as one could suspect that for any $a$ there exists some bad $y$ such that the original $p$ would be closer to $y$ than the `adjusted' $a$ is.%\par

%Now we want to consider only values in $Q$ that are adjusted to $\pi$, the class distribution of the dataset. We can then add another equality constraint so that the columns in any $a\in Q$ should average to $\pi$: $\frac{1}{n}\sum_{i=1}^n a_{i,j} = \pi_j \:\forall\: k  \in [1, k]$. This restricts the values in $Q$ only to adjusted sets of predictions. This set of $Q$ with the two equality constraints will be referred to as $Q^\star_\pi$ from now on for consistency's sake. The question now is if there exists a point $a\in Q^\star_\pi$ that is better than $p$ regardless of what the true posterior probabilities $q \in Q^\star_\pi$ are. Basically, are any of these points clearly better than $p$? It is not obvious that it is true, as one could suspect that for any $a$ there exists some bad $q$ such that the original $p$ would be closer to $q$ than the `adjusted' $a$ is.\par %it seems like the average divergence from $p$ to one of the points in this $Q^\star$ could be less than this proposed "best adjusted point", or maybe for some cases of $p$ and $Q^\star$.\par

Now we will define UGA and prove that it outputs adjusted predictions $a^\star$ that are indeed better than $p$, regardless of what the actual labels $y$ are. 
%What our following Theorem~\ref{thm:uga_coherence} will show is that an adjusted prediction with a guaranteed reduction in loss exists for every proper loss. , $a^\star \in Q^\star_\pi$, exists for any given $p$ and $Q^\star_\pi$ such that $a^\star$ has less average divergence than $p$ does to any matrix in $Q^\star_\pi$, including to $y$. That difference in average divergence from $p$ to $q$ and $a^\star$ to $q$ is also exactly the divergence from $p$ to $a^\star$. Our $a^\star$ calculating unbounded general adjuster (UGA) is defined as follows.

\begin{definition}[Unbounded General Adjuster (UGA)]
Consider a $k$-class classification task with a test dataset of $n$ instances, and let $d_\phi$ be a Bregman divergence. Then the \emph{unbounded general adjuster corresponding to $d_\phi$} is the function $\alpha^\star : \mathbb{R}^{n \times k} \times \mathbb{R}^k \to \mathbb{R}^{n \times k}$ defined as follows:
\begin{equation*}
%\begin{split}
\alpha^\star(p, \pi) = \argmin_{a\in Q_\pi} \frac{1}{n}\sum_{i=1}^n d_\phi (p_{i\cdot}, a_{i\cdot})
%\end{split}
\end{equation*}
%where $Q_\pi$ is defined by Eq.(\ref{eq:qpi}).
%...... =\{a\in \mathbb{R}^{n \times k}\mid\sum_{j=1}^k a_{i,j} = 1 ~\forall~ i  \in [1, n] \text{ and } \frac{1}{n}\sum_{i=1}^n a_{i,j} = \pi_j ~\forall~ j \in [1, k]\}$.
\end{definition}

The definition of UGA is correct in the sense that the optimisation task used to define it has a unique optimum. This is because it is a convex optimisation task, as will be explained in Section~\ref{sec:implementation}.
Intuitively, $Q_\pi$ can be thought of as an infinite hyperplane of adjusted predictions, also containing the unknown $y$. The original prediction $p$ is presumably not adjusted, so it does not belong to $Q_\pi$. UGA essentially `projects' $p$ to the hyperplane $Q_\pi$, in the sense of finding $a$ in the hyperplane which is closest from $p$ according to $d_\phi$, see the diagram in Figure~\ref{fig:diagram}.

\begin{figure}[t] 
  \centering
  \includegraphics[scale=0.3]{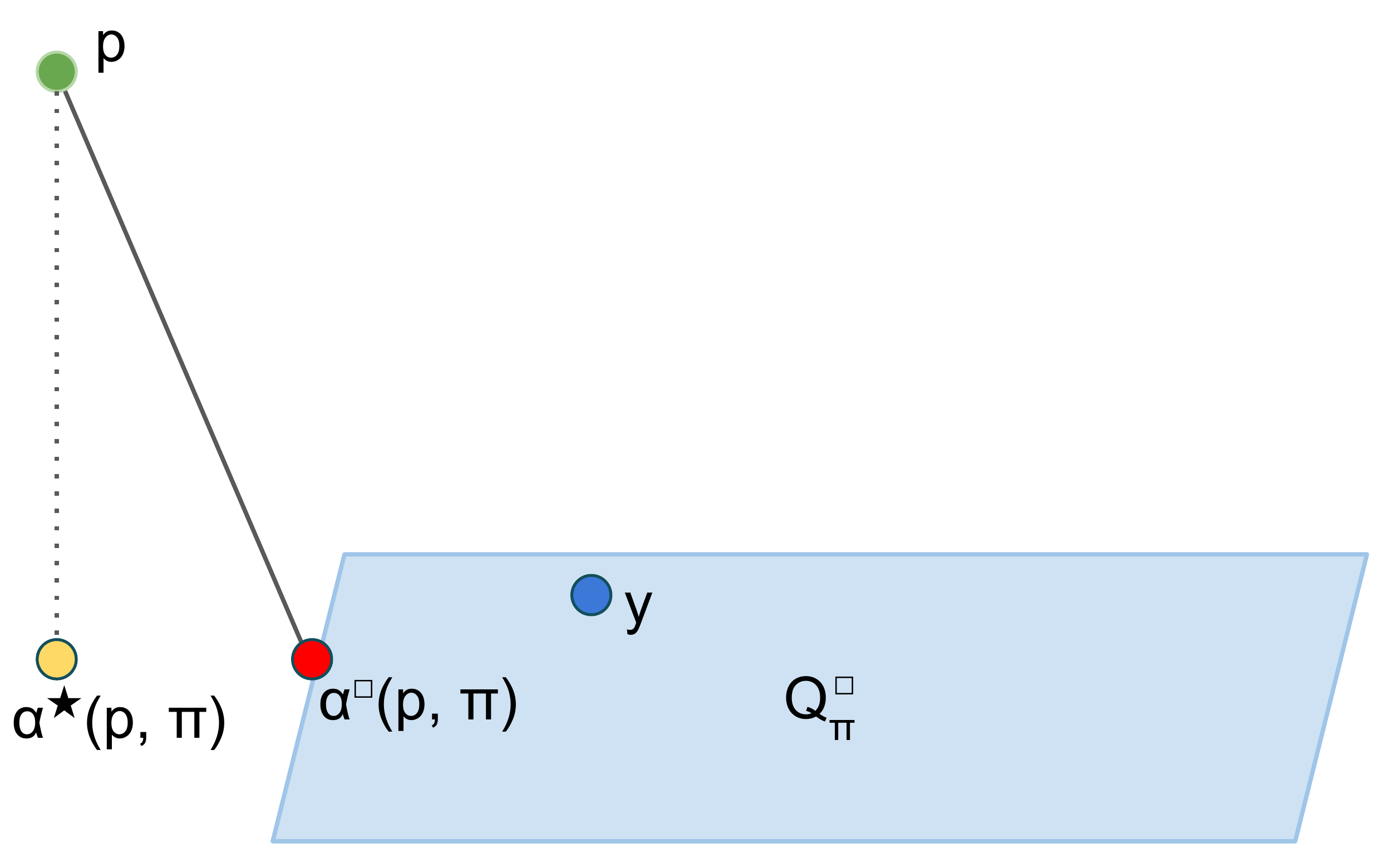}\quad
  \caption{A schematic explanation with $\alpha^\star(p,\pi)$ of UGA and $\alpha^\square(p,\pi)$ of BGA.} \label{fig:diagram}
\end{figure}

The following theorem guarantees that the loss is reduced after applying UGA by showing that UGA is coherent with its Bregman divergence.

\begin{theorem}
Let $\alpha^\star$ be the unbounded general adjuster corresponding to the Bregman divergence $d_\phi$.
Then $\alpha^\star$ is coherent with $d_\phi$.
\end{theorem}
\begin{proof}
The proof of this and following theorems are given in the Supplementary.
\end{proof}

The next theorem proves that UGA is actually the one and only adjustment procedure that decomposes in the sense of Theorem~\ref{thm:decomp}. Therefore, UGA coincides with additive and multiplicative adjustment on Brier score and log-loss, respectively.

\begin{theorem}
Let $d_\phi$ be a Bregman divergence, let $p$ be a set of predictions, and $\pi$ be a class distribution over $k$ classes. Suppose $a\in Q_\pi$ is such that for any $y\in Q_\pi$ the decomposition of Eq.(\ref{eq:decomp}) holds.
%\begin{align*}
%\frac{1}{n}\sum_{i=1}^n d_\phi(p_{i\cdot},y_{i\cdot}) =
%\frac{1}{n}\sum_{i=1}^n d_\phi(p_{i\cdot},a_{i\cdot}) +
%\frac{1}{n}\sum_{i=1}^n d_\phi(a_{i\cdot},y_{i\cdot})
%\end{align*}
Then $a=\alpha^\star(p,\pi)$.
\end{theorem}

As explained in the example of additive adjustment (which is UGA for Brier score), some adjusted predictions can get out from the range $[0,1]$. 
%From the description of additive adjustment as given at the end of Section~\ref{sec:background} one can see that there is nothing ensuring that none of the values in the adjusted probability vectors would be negative. This is easy to see, as when the proportion of a class is over-estimated by the model, then during additive adjustment all probabilities for this class will be decreased, even the ones that were already $0$ (or very close to it) to start with. 
It is clear that a prediction involving negative probabilities cannot be optimal. In the following section we propose the Bounded General Adjuster (BGA) which does not satisfy the decomposition property but is guaranteed to be at least as good as UGA.

\subsection{Bounded General Adjustment}

For a given class distribution $\pi$, let us constrain the set of all possible adjusted predictions $Q_\pi$ further, by requiring that all probabilities are non-negative: 
$$Q^\square_\pi=\{a\in Q_\pi\mid a_{i,j}\geq 0\ \text{for}\ i=1,\dots,n\ \text{and}\ j=1,\dots,k\}$$
%)  \in [1, n] \times [1, k]\}$.
We now propose our bounded general adjuster (BGA), which outputs predictions within $Q^\square_\pi$.

\begin{definition}[Bounded General Adjuster (BGA)]
Consider a $k$-class classification task with a test dataset of $n$ instances, and let $d_\phi$ be a Bregman divergence. Then the \emph{bounded general adjuster corresponding to $d_\phi$} is the function $\alpha^\square : [0,1]^{n \times k} \times [0,1]^k \to [0,1]^{n \times k}$ defined as follows:
\begin{equation*}
%\begin{split}
\alpha^\square(p, \pi) = \argmin_{a\in Q^\square_\pi} \frac{1}{n}\sum_{i=1}^n d_\phi (p_{i\cdot}, a_{i\cdot})
%\end{split}
\end{equation*}
%where $Q_\pi$ is defined by Eq.(\ref{eq:qpi}).
%...... =\{a\in \mathbb{R}^{n \times k}\mid\sum_{j=1}^k a_{i,j} = 1 ~\forall~ i  \in [1, n] \text{ and } \frac{1}{n}\sum_{i=1}^n a_{i,j} = \pi_j ~\forall~ j \in [1, k]\}$.
\end{definition}

%Let $d_\phi$ be our Bregman divergence, $p$ be a set of predictions, and $\pi$ is the class distribution. Given these we define our bounded generalized adjustment function $\alpha^\square : [0,1]^{n \times k} \times [0,1]^k \to [0,1]^{n \times k}$:
%\begin{equation*}
%\begin{split}
%\alpha^\square(p, \pi) = \argmin_{a\in Q^\square\pi} \sum_{i=1}^n d_\phi (p_i, a_i)
%\end{split}
%\end{equation*}
%\end{definition}

Similarly as for UGA, the correctness of BGA is guaranteed by the convexity of the optimisation task, as shown in Section~\ref{sec:implementation}.
BGA solves almost the same optimisation task as UGA, except that instead of considering the whole hyperplane $Q_\pi$ it finds the closest $a$ within a bounded subset $Q^\square_\pi$ within the hyperplane. Multiplicative adjustment is the BGA for log-loss, because log-loss is not defined at all outside the $[0,1]$ bounds, and hence the UGA for log-loss is the same as the BGA for log-loss. 
The following theorem shows that there is a guaranteed reduction of loss after BGA, and the reduction is at least as big as after UGA.

\begin{theorem}
Let $d_\phi$ be a Bregman divergence, let $p$ be a set of predictions, and $\pi$ be a class distribution over $k$ classes. Then for any $y\in Q^\square_\pi$ the following holds:
\begin{align*}
&\sum_{i=1}^n (d_\phi(p_{i\cdot},y_{i\cdot}) - d_\phi(a_{i\cdot}^\square, y_{i\cdot})) \\
&\geq\sum_{i=1}^n d_\phi(p_{i\cdot}, a_{i\cdot}^\square) \geq 
\sum_{i=1}^n d_\phi(p_{i\cdot}, a_{i\cdot}^\star)=
\sum_{i=1}^n (d_\phi(p_{i\cdot},y_{i\cdot}) - d_\phi(a_{i\cdot}^\star, y_{i\cdot}))
\end{align*}
%\frac{1}{n}\sum_{i=1}^n d_\phi(p_{i\cdot},y_{i\cdot}) =
\end{theorem}

%Let $d_\phi : ri(S) \times S \mapsto [0, \inf )$ be our Bregman divergence, let $p$ be a set of predictions and $\pi$ be a class distribution. Then for any $q\in Q^\square_\pi$ the following holds:
%\begin{equation}
%\sum_{i=1}^n (d_\phi(p_i,q_i) - d_\phi(a^\square, q_i)) \geq \sum_{i=1}^n d_\phi(p_i, a_i^\square) \geq
%\end{equation}
%\begin{equation}
%\sum_{i=1}^n d_\phi(p_i, a_i^\star)=\sum_{i=1}^n (d_\phi(p_i,q_i) - d_\phi(a^\star, q_i))
%\end{equation}
%\end{theorem}

%This theorem says that our adjustment function guarantees an expected reduction of divergence and that reduction will be at least as much as $\sum_{i=1}^n d_\phi(p_i, a^\square_i)$. 
%This actually produces even a better reduction in loss than the unbounded adjuster, but loses its ability to decompose.\par
Note that the theorem is even more general than we need and holds for all $y\in Q^\square_\pi$, not only those $y$ which represent label matrices.
A corollary of this theorem is that the BGA for Brier score is a new adjustment method dominating over additive adjustment in reducing Brier score.
In practice, all practitioners should prefer BGA over UGA when looking to adjust their classifiers. Coherence and decomposition are interesting from a theoretical perspective but from a loss reduction standpoint, BGA is superior to UGA.

\subsection{Implementation} \label{sec:implementation}

Both UGA and BGA are defined through optimisation tasks, which can be shown to be convex. First, the objective function is convex as a sum of convex functions (Bregman divergences are convex in their second argument \cite{bauschke2001joint}). Second, the equality constraints that define $Q_\pi$ are linear, making up a convex set. Finally, the inequality constraints of $Q^\square_\pi$ make up a convex set, which after intersecting with $Q_\pi$ remains convex. These properties are sufficient \cite{boyd2004convex} to prove that both the UGA and BGA optimisation tasks are convex.%\par

%With $Q^\star_\pi$ as our search space we have a convex set. A feasible region is convex if its equality constraints are linear or affine. %[?]. 
%The equality constraints are unweighted sums of real numbers, so they are indeed linear. With $Q^\square_\pi$ an extra inequality constraint is added to the search space. This can be represented as the intersection of $Q^\star_\pi$ and $[0,1]^{n \times k}$, which is also obviously a convex set. The intersection of two convex sets is also a convex set,% [?], 
%so BGA's domain is also convex.\par 
%For UGA and BGA, both are minimizing the same function, the sum of Bregman divergences. This sum is a convex function because each individual Bregman divergence is convex and positive weighted sums of convex functions result in a convex functions. So the requirement that the function is convex is satisfied for both of them, giving us the requirements to implement these functions with convex optimizers.\par
UGA has only equality constraints, so Newton's method works fine with it. For Brier score there is a closed form solution \cite{kull2015novel} of simply adding the difference between the new distribution and the old distribution for every set of k probabilities.
%
%\begin{theorem}[BGA is Convex Optimizable]
%Given a set of predictions $p \in \mathbb{R}^{n \times k}$ and a class distribution $\pi \in \mathbb{R}^k$, $\alpha^\square$ is a convex optimization problem.
%\end{theorem}
%\begin{proof}
%To be a convex optimization problem both the objective and constraints need to satisfy the following inequality\cite{boyd2004convex}.
%$$f_i(\alpha x + \beta y) \leq \alpha f_i(x) + \beta f_i(y)$$
%
%\end{proof}
%
%Since BGA is convex and UGA is a less constrained version of BGA, it follows that UGA is also solvable through convex optimization.
%
BGA computations are a little more difficult due to inequality constraints, therefore requiring interior point methods \cite{boyd2004convex}. %These involve transforming the objective using an indicator function that acts as the barrier defined by the inequality constraints so that the constraints can just be described as inequality constraints and then applying Newton's method.\par
While multiplicative adjustment is for log-loss both BGA and UGA at the same time, it is easier to calculate it as a UGA, due to not having inequality constraints.

%
%then it is e (as discussed for log-loss at
%It is worth stating noting that although BGA for log-loss KL divergence is naturally bounded between 0 and 1. The divergence and magnitude of its gradient go to infinity at the bounds. This means BGA is equivalent to UGA for KL divergence, and the more efficient Euclidean methods can be used instead of interior point methods. %However, we also have an iterative algorithm to solve to it.

%% file: sections/4_experiments.tex
\section{Experiments}
%#\begin{enumerate}
%    \item Clarify the experiments (describe the shifts applied by formulas, not just words)
%    \item Clarify difference between "types of shift" and how we are shifting.
%    \item Give cases of things that aren't "just PPS" but you still know final class distribution
%    \item Make clear that our "error rate" is the mis-specification of the class distribution
%\end{enumerate}

\subsection{Experimental Setup}

While our theorems provide loss reduction guarantees when the exact class distribution is known, this is rarely something to be expected in practice. Therefore, the goal of our experiments was to evaluate the proposed adjustment methods in the setting where the class distribution is known approximately. For loss measures we selected Brier score and log-loss, which are the two most well known proper losses. As UGA is dominated by BGA, we decided to evaluate only BGA, across a wide range of different types of dataset shift, across different classifier learning algorithms, and across many datasets. We compare the results of BGA with the prior probability adjuster (PPA) introduced in Section~\ref{sec:background}, as this is to our knowledge the only existing method that inputs nothing else than the predictions of the model and the shifted class distribution. As reviewed in 
%\cite{pan2010survey} and 
\cite{weiss2016survey}, other existing transfer learning methods need either the information about the features or a limited set of labelled data from the shifted context.
%We decided to test adjustment works on top of probabilistic model predictions, it is important to consider a variety of  It is also unclear whether BGA (with the appropriate matching loss score) is generally the best choice for adjusting. We ran a series of experiments to address these concerns by comparing the performance of BGA for Brier score, BGA for log loss, and the previously mentioned prior probability adjustment (PPA), under various shift patterns, amounts of shift, and amounts of error in the estimation of post-shift class distribution. UGA was not experimented on as we proved that it is outperformed by BGA. Our experiments are entirely reproducible. All code is included in the supplementary material.\par

To cover a wide range of classifiers and datasets, we opted for using OpenML \cite{OpenML2013}, which contains many datasets, and for each dataset many runs of different learning algorithms. For each run OpenML provides the predicted class probabilities in 10-fold cross-validation. 
%While the predictions in 10-fold cross-validation are obtained with 10 different models trained on different subsets of folds, the models are usually very similar and behave statistically similarly enough to be combined into a single prediction matrix $p$ for our purpose of evaluating adjustment.
As the predictions in 10-fold cross-validation are obtained with 10 different models trained on different subsets of folds, we compiled the prediction matrix $p$ from one fold at a time. %used predictions on each fold as a separate prediction matrix $p$ to be adjusted. % for our purpose of evaluating adjustment.
%using their R library API \cite{R2015,OpenMLR2017,RStudio2016}. OpenML is organized into tasks, a grouping of data and description of the goal (e.g. minimizing log-loss or maximizing accuracy for classification, minimizing MSE for regression, etc.), among other things. Each task has user-submitted runs associated with it. The runs are essentially the user's model's predictions for each data point in the task. Runs include folds so that the predictions aren't done on instances that were used during training. That means, in total, there is a set of predictions for a different model for every fold in every run in every classification task.\par
%
%We used the R library API of OpenML \cite{R2015,OpenMLR2017,RStudio2016} 
From OpenML we downloaded all user-submitted sets of predictions for both binary and multiclass (up to eight classes) classification tasks, restricting ourselves to tasks with the number of instances in the interval of $[2000,1000000]$. Then we discarded every dataset that included a predicted score outside the range $(0,1)$. To emphasize, we did not include runs which contain a $0$ or a $1$ anywhere in the predictions, since log-loss becomes infinite in case of errors with full confidence. We discarded datasets with less than $500$ instances and sampled datasets with more than $1000$ instances down to $1000$ instances. This left us with 590 sets of predictions, each from a different model. 
%These 3112 sets of predictions come from 312 different runs from 183 different classification tasks. 
These 590 sets of predictions come from 59 different runs from 56 different classification tasks. 
The list of used datasets and the source code for running the experiments is available in the Supplemental Material. %\par
%\todo{needs to be clarified how many after discarding}

%Now that we had our datasets, we took each one and discarded it if it had less than 500 instances and down-sampled it to 1000 instances if it had more than 1000 instances.\par

\paragraph{Shifting.}

For each dataset we first identified the majority class(es). After sorting the classes by size decreasingly, the class(es) $1,\dots,m$ were considered as majority class(es), where $j$ was the smallest possible integer such that $\pi_1+\dots+\pi_m>0.5$. We refer to other class(es) as minority class(es).
We then created $4$ variants of each dataset by artificially inducing shift in four ways. Each of those shifts has a parameter $\varepsilon\in[0.1,0.5]$ quantifying the amount of shift, and $\varepsilon$ was chosen uniformly randomly and independently for each adjustment task.

The first method induces prior probability shift by undersampling the majority class(es), reducing their total proportion from $\pi_1+\dots+\pi_m$ to $\pi_1+\dots+\pi_m-\varepsilon$. The second method induces a variety of concept shift by selecting randomly a proportion $\varepsilon$ of instances from majority class(es) and changing their labels into uniformly random minority class labels. The third method induces covariate shift by deleting within class $m$ the proportion $\varepsilon$ of the instances with the lowest values of the numeric feature which correlates best with this class label. The fourth method was simply running the other three methods all one after another, which produces an other type of shift. %\par
%\todo{in which order were 3 methods applied?}
%\todo{give details about covariate shift generation}
%We randomized the "amount" of shift by randomly and uniformly selecting a real number between [0.1, 0.5] every time we ran a new method. By "amount", we are referring to the proportion of instances we deleted/undersampled from the dataset for methods one and three, and flipped the class of for method two.\par

\paragraph{Approximating the New Class Distribution.}

It is unlikely that a practitioner of adjustment would know the exact class distribution of a shifted dataset. To investigate this, we ran our adjustment algorithms on our shifted datasets with not only the exact class distribution, but also eight `estimations' of the class distribution obtained by artificially modifying the correct class distribution $(\pi_1,\dots,\pi_k)$ into $(\pi_1+\delta,\dots,\pi_m+\delta,\pi_{m+1}-\delta',\dots,\pi_k-\delta'$, where $\delta$ was one of eight values $+0.01, -0.01, +0.02, -0.02, +0.04, -0.04, +0.08, -0.08$, and $\delta'$ was chosen to ensure that the sum of class proportions adds up to $1$. %subtracting corresponding amounts from the minority class(es) (making sure that the . 
%By "big classes", we are referring to the classes with the biggest proportion of the dataset and when added together make up the majority of the dataset. The other "small classes" are then proportionately reduced so the "estimated" class distribution summed to 1. 
If any resulting class proportion left the [0,1] bounds, then the respective adjustment task was skipped.%\par
In total, we recorded results for 17527 adjustment tasks resulting from combinations of dataset fold, shift amount, shift method, and estimated class distribution.%\par

\paragraph{Adjustment.}

For every combination of shift and for the corresponding nine different class distribution estimations, we adjusted the datasets/predictions using the three above-mentioned adjusters: Brier-score-minimizing-BGA, log-loss-minimizing-BGA, and PPA. PPA has a simple implementation, but for the general adjusters we used the CVXPY library \cite{cvxpy} to perform convex optimization. 
%CVXPY uses external optimizers to run the actual optimizing step. 
For Brier-score-minimizing-BGA, the selected method of optimization was OSQP (as part of the CVXPY library). 
%\cite{stellato2018osqp}. 
For log-loss-minimizing-BGA, we used the ECOS optimizer
%\cite{domahidi2013ecos} 
with the SCS 
%\cite{o2016scs} 
optimizer as backup (under rare conditions the optimizers could numerically fail, occurred 30 times out of 17527). For both Brier score and log loss, we measured the unadjusted loss and the loss after running the dataset through the aforementioned three adjusters. 

\begin{figure}[t] 
  \centering
  \includegraphics[height=5.7cm]{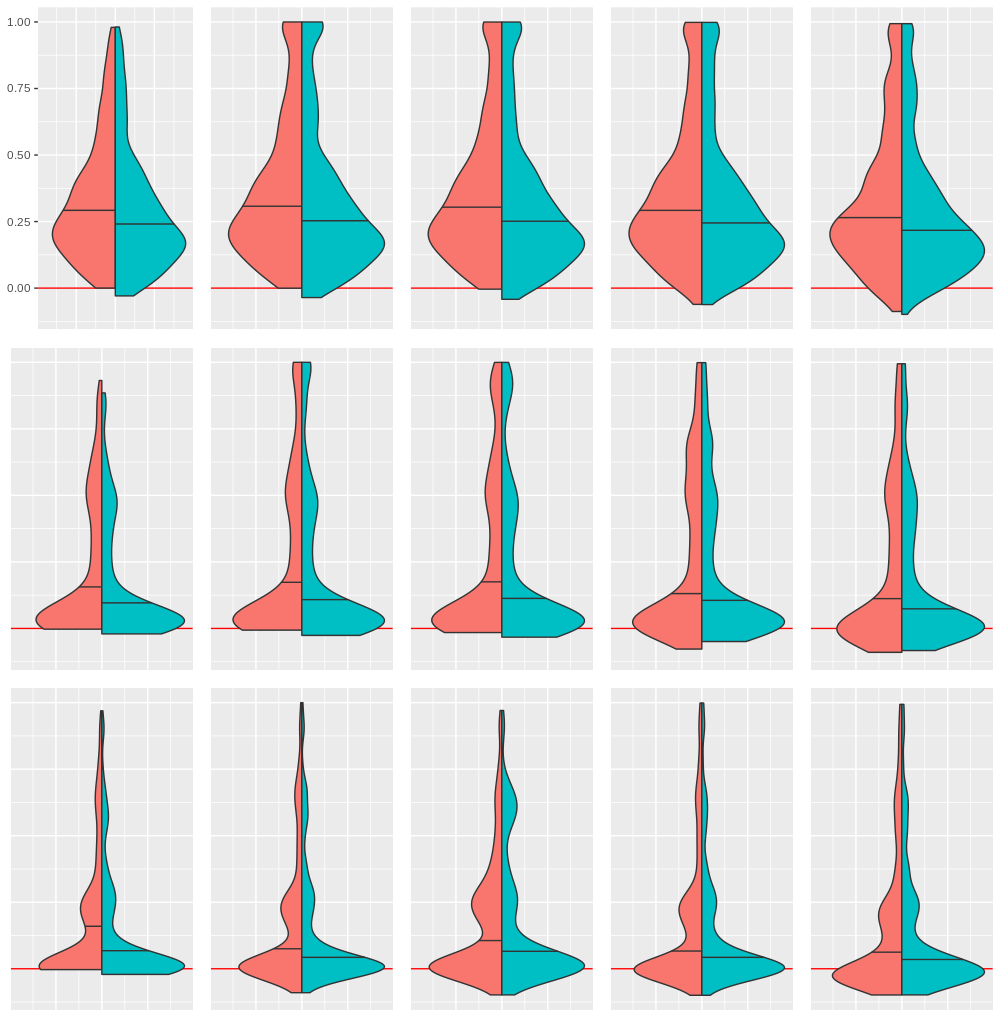}\quad\:
  \includegraphics[height=5.7cm]{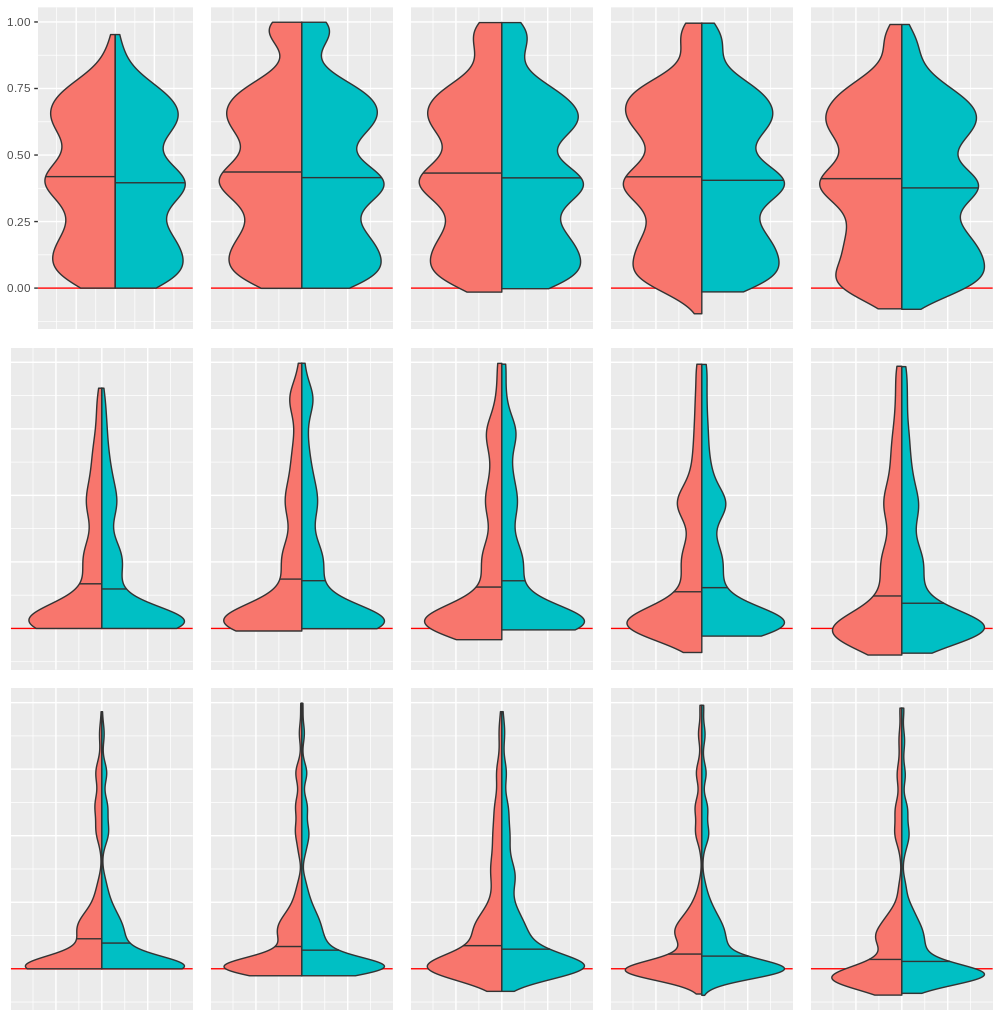}
  \caption{The reduction in Brier score (left figure) and log-loss (right figure) after BGA adjustment (left side of the violin) and after PPA adjustment (right side of the violin). The rows correspond to different amounts of shift (with high shift at the top and low at the bottom). The columns correspond to amount of induced error in class distribution estimation, starting from left: 0.00, 0.01, 0.02, 0.04 and 0.08.}\label{fig:res1} 
\end{figure}

\begin{figure}[t]
  \centering
  \includegraphics[height=5.6cm]{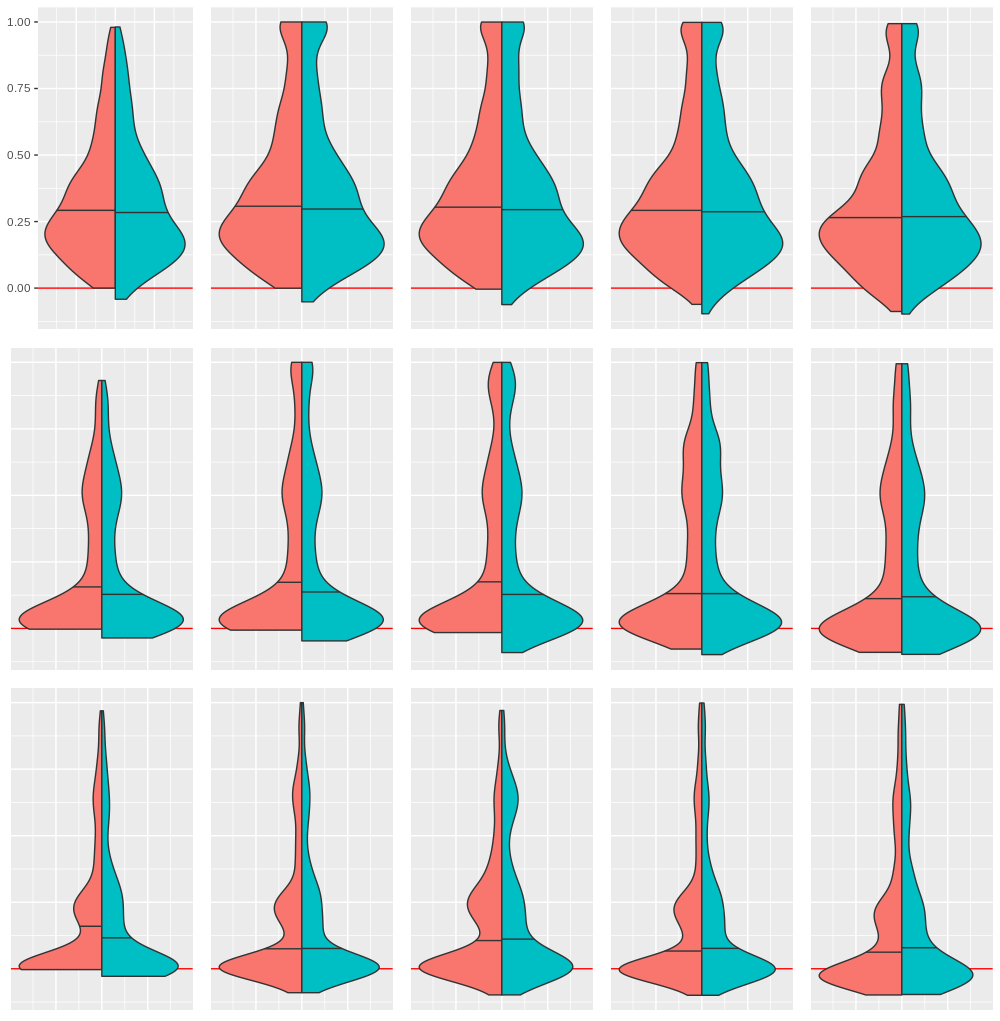}\quad\:
  \includegraphics[height=5.6cm]{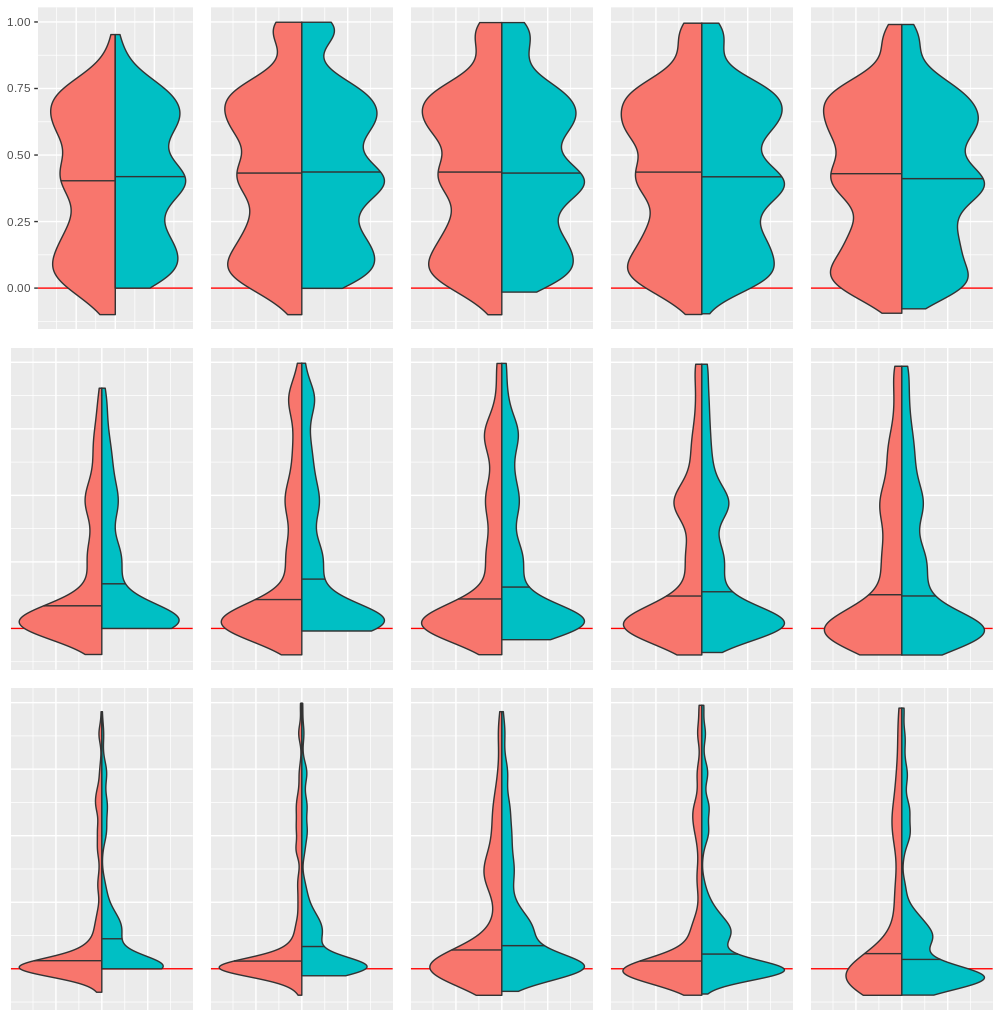}
  \caption{The reduction in Brier score (left figure) and log-loss (right figure) after BGA adjustment to reduce Brier score (left side of the violin) and after BGA to reduce log-loss (right side of the violin). The rows correspond to different amounts of shift (high at the top and low at the bottom). The columns correspond to amount of induced error in class distribution estimation, starting from left: 0.00, 0.01, 0.02, 0.04 and 0.08.}\label{fig:res2} 
\end{figure}

\subsection{Results}

On different datasets the effects of our shifting procedures vary and thus we have categorized the shifted datasets into 3 equal-sized groups by the amount of squared Euclidean distance between the original and new class distributions (high, medium and low shift). Note that these are correlated to the shift amount parameter $\varepsilon$, but not determined by it.
Figures~\ref{fig:res1} and \ref{fig:res2} both visualise the loss reduction after adjustment in proportion to the loss before adjustment. In these violin plots the part of distributions above $0$ stands for reduction of loss and below $0$ for increased loss after adjustment. For example, proportional reduction value $0.2$ means that 20\% of the loss was eliminated by adjustment. The left side of the left-most violins in Figure~\ref{fig:res1} show the case where BGA for Brier score is evaluated on Brier score (with high shift at the top row and low at the bottom). Due to guaranteed reduction in loss the left sides of violins are all above $0$. In contrast, the right side of the same violins shows the effect of PPA adjustment, and PPA can be seen to sometimes increase the loss, while also having lower average reduction of loss (the horizontal black line marking the mean is lower). When the injected estimation error in the class distribution increases (next 4 columns of violins), BGA adjustment can sometimes increase the loss as well, but is on average still reducing loss more than PPA in all of the violin plots. 
%. the class distribution is known exactly, and hence the left subfigure presents the results on each of the 3 categories (rows) and those groups as a two-sided violin plot showing the reduction in Brier score after adjustment: left side uses BGA adjustment and the right side PPA adjustment.
Similar patterns of results can be seen in the right subfigure of Figure~\ref{fig:res1}, where BGA for log-loss is compared with PPA, both evaluated on log-loss. The mean proportional reduction of loss by BGA is higher than by PPA in 13 out of 15 cases. The bumps in some violins are due to using 4 different types of shift. 

%For each of the 3 categories we visualise

%We first grouped the results into 15 groups: 5 groups according to the induced error into our estimation of the class distribution (no error, 0.01 and -0.01, 0.02 and -0.02, 0.04 and -0.04, 0.08 and -0.08) %and splitting each of these into 3 equal-sized subgroups based on the amount of shift (high, medium and low squared Euclidean distance). 
%Figure~\ref{fig:res1} left presents the results on each of those groups as a two-sided violin plot showing the reduction in Brier score after adjustment: left side uses BGA adjustment and the right side PPA adjustment. Figure~\ref{fig:res1} right presents the same for reduction in log-loss. Both figures demonstrate that the loss gets reduced most of the time, seen from the violins being mostly above 0. Usually BGA outperforms PPA and the effect is stronger for smaller amounts of induced error in class distribution estimation. PPA outperforms BGA in log-loss on very high levels of induced error.

Figure~\ref{fig:res2} demonstrates the differences between BGA aiming to reduce Brier score (left side of each violin) and BGA to reduce log loss (right side of each violin), evaluated on Brier score (left subfigure) and log-loss (right subfigure). As seen from the right side of the leftmost violins, BGA aiming to reduce the wrong loss (log-loss) can actually increase loss (Brier score), even if the class distribution is known exactly. 
%Similar case is seen in the left side of the leftmost violins in the right subfigure, where BGA aiming to reduce Brier score at the same time can increase log-loss. 
Therefore, as expected, it is important to adjust by minimising the same divergence that is going to be used to test the method.

%% file: sections/5_conclusion.tex
\section{Conclusion}

In this paper we have constructed a family BGA of adjustment procedures aiming to reduce any proper loss of probabilistic classifiers after experiencing dataset shift, using knowledge about the class distribution. We have proved that the loss is guaranteed to reduce, if the class distribution is known exactly. %Our experiments have showed that even if the class distribution is known only approximately, then BGA often provides reduction in loss. 
According to our experiments, BGA adjustment to an approximated class distribution often still reduces loss more than prior probability adjustment.

%% file: supplement.tex
\date{}
\title{Supplemental Material for\\ Shift Happens: Adjusting Classifiers}
\author{Theodore James Thibault Heiser \and
Mari-Liis Allikivi \and
Meelis Kull}
\authorrunning{T. Heiser et al.}
\institute{Institute of Computer Science, University of Tartu, Tartu, Estonia\\
\email{\{mari-liis.allikivi,meelis.kull\}@ut.ee}}
\maketitle

\section{Introduction}

This supplemental material of the paper \emph{Shift Happens: Adjusting Classifiers} first lists all the definitions of the paper, followed by theorems and proofs and also the lemmas needed to complete these proofs. The source code and the list of IDs of OpenML tasks and runs that enter our experiments have been provided as file \texttt{source\_code.zip} together with the current document.

\section{Definitions}

\begin{definition}[Proper Scoring Rule (or Proper Loss)]
In a $k$-class classification task a loss function $f:[0,1]^k\times\{0,1\}^k\to\mathbb{R}$ is called a \emph{proper scoring rule} (or \emph{proper loss}), if for any probability vectors $p,q\in[0,1]^k$ with $\sum_{i=1}p_i=1$ and $\sum_{i=1}q_i=1$ the following inequality holds:
%\begin{align*}
$$\expect_{Y \sim q}[f(q,Y)] \leq \expect_{Y \sim q}[f(p,Y)]$$ %\\[-0.3cm]
%\end{align*}
where $Y$ is a one-hot encoded label randomly drawn from the categorical distribution over $k$ classes with class probabilities represented by vector $q$. 
%Given a loss function $f$, a prediction of the probability distribution of labels $p$, a random variable $Y$ representing a label, and the true probability distribution of labels $q$, we define $f$ to be a proper scoring rule if
The loss function $f$ is called \emph{strictly proper} if the inequality is strict for all $p\neq q$.
\end{definition}

\begin{definition}[Bregman Divergence]
Let $\phi : \Omega \to \mathbb{R}$ be a strictly convex function defined on a convex set $\Omega \subseteq \mathbb{R}^k$ such that $\phi$ is differentiable on the relative interior of $\Omega$, $ri(\Omega)$. The Bregman divergence $d_\phi : ri(\Omega) \times \Omega \to [0, \infty )$ is defined as 
%\vspace{-0.3pt}
%\begin{align*}
$$d_\phi (p,q) = \phi (q) - \phi(p) - \langle q - p , \nabla \phi (p) \rangle$$ %\\[-0.5cm]
%\end{align*}
\end{definition}

\begin{definition}[Adjusted Predictions]
Let $p\in[0,1]^{n\times k}$ be the predictions of a probabilistic $k$-class classifier on $n$ instances and let $\pi\in[0,1]^k$ be the actual class distribution on these instances. We say that \emph{predictions $p$ are adjusted on this dataset}, if the average prediction is equal to the class proportion for every class $j$, that is
$\frac{1}{n}\sum_{i=1}^n p_{ij}=\pi_j$.
\end{definition}

\begin{definition}[Additive Adjustment]
\emph{Additive adjustment} is the function $\alpha_{+}:[0,1]^{n\times k}\times [0,1]^k\to [0,1]^{n\times k}$ which takes in the predictions of a probabilistic $k$-class classifier on $n$ instances and the actual class distribution $\pi$ on these instances, and outputs adjusted predictions $a=\alpha_{+}(p,\pi)$ defined as
%\begin{align*}
$a_{i\cdot}=p_{i\cdot}+(\varepsilon_1,\dots,\varepsilon_k)$
%\end{align*}
where $a_{i\cdot}=(a_{i1},\dots,a_{ik})$, $p_{i\cdot}=(p_{i1},\dots,p_{ik})$, and $\varepsilon_j=\pi_j-\frac{1}{n}\sum_{i=1}^n p_{ij}$ for each class $j\in\{1,\dots,k\}$.
\end{definition}

\begin{definition}[Adjustment Procedure]
\emph{Adjustment procedure} is any function $\alpha:[0,1]^{n\times k}\times [0,1]^k\to [0,1]^{n\times k}$ which takes as arguments the predictions $p$ of a probabilistic $k$-class classifier on $n$ instances and the actual class distribution $\pi$ on these instances, such that for any $p$ and $\pi$ the output predictions $a=\alpha(p,\pi)$ are adjusted, that is $\frac{1}{n}\sum_{i=1}^n a_{ij}=\pi_j$ for each class $j\in\{1,\dots,k\}$.
\end{definition}

\begin{definition}[Bounded Adjustment Procedure]
An adjustment procedure $\alpha:[0,1]^{n\times k}\times [0,1]^k\to [0,1]^{n\times k}$ is \emph{bounded}, if for any $p$ and $\pi$ the output predictions $a=\alpha(p,\pi)$ are in the range $[0,1]$, that is $a_{ij}\in[0,1]$ for all $i,j$.
\end{definition}

\begin{definition}[Multiplicative Adjustment]
\emph{Multiplicative adjustment} is the function $\alpha_{*}:[0,1]^{n\times k}\times [0,1]^k\to [0,1]^{n\times k}$ which takes in the predictions of a probabilistic $k$-class classifier on $n$ instances and the actual class distribution $\pi$ on these instances, and outputs adjusted predictions $a=\alpha_{*}(p,\pi)$ defined as 
%\begin{align*}
$a_{ij}=\frac{w_j p_{ij}}{z_i}$,
%\end{align*}
where $w_1,\dots,w_k\geq 0$ are real-valued weights chosen based on $p$ and $\pi$ such that the predictions $\alpha_{*}(p,\pi)$ would be adjusted, and $z_i$ are the renormalisation factors defined as $z_i=\sum_{j=1}^k w_j p_{ij}$. 
\end{definition}

\begin{definition}[Coherence of Adjustment Procedure and Bregman Divergence \cite{kull2015novel}]
Let $\alpha:[0,1]^{n\times k}\times [0,1]^k\to [0,1]^{n\times k}$ be an adjustment procedure and $d_\phi$ be a Bregman divergence. 
Then $\alpha$ is called to be coherent with $d_\phi$ if and only if for any predictions $p$ and class distribution $\pi$ the following holds for all $i=1,\dots,n$ and $j,j'=1,\dots,k$:
$$\left(d_\phi(a_{i\cdot},c_j)-d_\phi(p_{i\cdot},c_j)\right)
 -\left(d_\phi(a_{i\cdot},c_{j'})-d_\phi(p_{i\cdot},c_{j'})\right)
 =const_{j,j'}$$
where $const_{j,j'}$ is a quantity not depending on $i$, and where $a=\alpha(p,\pi)$ and $c_j$ is a one-hot encoded vector corresponding to class $j$ (with $1$ at position $j$ and $0$ everywhere else).
\end{definition}

\begin{definition}[Unbounded General Adjuster (UGA)]
Consider a $k$-class classification task with a test dataset of $n$ instances, and let $d_\phi$ be a Bregman divergence. Then the \emph{unbounded general adjuster corresponding to $d_\phi$} is the function $\alpha^\star : \mathbb{R}^{n \times k} \times \mathbb{R}^k \to \mathbb{R}^{n \times k}$ defined as follows:
\begin{equation*}
%\begin{split}
\alpha^\star(p, \pi) = \argmin_{a\in Q_\pi} \frac{1}{n}\sum_{i=1}^n d_\phi (p_{i\cdot}, a_{i\cdot})
%\end{split}
\end{equation*}
%where $Q_\pi$ is defined by Eq.(\ref{eq:qpi}).
%...... =\{a\in \mathbb{R}^{n \times k}\mid\sum_{j=1}^k a_{i,j} = 1 ~\forall~ i  \in [1, n] \text{ and } \frac{1}{n}\sum_{i=1}^n a_{i,j} = \pi_j ~\forall~ j \in [1, k]\}$.
\end{definition}

\begin{definition}[Bounded General Adjuster (BGA)]
Consider a $k$-class classification task with a test dataset of $n$ instances, and let $d_\phi$ be a Bregman divergence. Then the \emph{bounded general adjuster corresponding to $d_\phi$} is the function $\alpha^\square : [0,1]^{n \times k} \times [0,1]^k \to [0,1]^{n \times k}$ defined as follows:
\begin{equation*}
%\begin{split}
\alpha^\square(p, \pi) = \argmin_{a\in Q^\square_\pi} \frac{1}{n}\sum_{i=1}^n d_\phi (p_{i\cdot}, a_{i\cdot})
%\end{split}
\end{equation*}
%where $Q_\pi$ is defined by Eq.(\ref{eq:qpi}).
%...... =\{a\in \mathbb{R}^{n \times k}\mid\sum_{j=1}^k a_{i,j} = 1 ~\forall~ i  \in [1, n] \text{ and } \frac{1}{n}\sum_{i=1}^n a_{i,j} = \pi_j ~\forall~ j \in [1, k]\}$.
\end{definition}

\section{Theorems and Lemmas with Proofs}

\begin{theorem}[Decomposition of Bregman Divergences \cite{kull2015novel}]\label{thm:decomp}
Let $d_\phi$ be a Bregman divergence and let $\alpha:[0,1]^{n\times k}\times [0,1]^k\to [0,1]^{n\times k}$ be an adjustment procedure coherent with $d_\phi$.
Then for any predictions $p$, one-hot encoded true labels $y\in\{0,1\}^{n\times k}$ and class distribution $\pi$ (with $\pi_j=\frac{1}{n}\sum_{i=1}^n y_{ij}$) the following decomposition holds:
\begin{align} \label{eq:decomp}
\frac{1}{n}\sum_{i=1}^n d_\phi(p_{i\cdot},y_{i\cdot}) =
\frac{1}{n}\sum_{i=1}^n d_\phi(p_{i\cdot},a_{i\cdot}) +
\frac{1}{n}\sum_{i=1}^n d_\phi(a_{i\cdot},y_{i\cdot})
\end{align}
\end{theorem}
\begin{proof}
Proof given in cited article \cite{kull2015novel}.
\end{proof}

\begin{lemma} \label{lem:one}
Let $d_\phi : ri(\Omega) \times \Omega \to \mathbb{R}$ be a Bregman divergence. 
Then for any $p,q\in ri(\Omega)$ the following holds:
$$\nabla_q d_\phi (p,q) = \nabla\phi (q) - \nabla\phi (p),$$
%where $\nabla_q$ notates the gradient's variables taken with respect to the variables of $q$.
where $\nabla_q$ notates the gradient with respect to vector $q$.
\end{lemma}
\begin{proof} By the definition of Bregman divergence,
$$d_\phi (p, q) = \phi(q) - \phi(p) - \langle q - p, \nabla\phi (p) \rangle .$$
The required result follows by taking $\nabla_q$ of each side and simplifying: 
\begin{equation*}
\begin{split}
\nabla_q d_\phi (p,q) &= \nabla_{q} (\phi(q) - \phi(p) - \langle q - p, \nabla\phi (p) \rangle ) \\
    & = \nabla_{q} \phi(q) - \nabla_{q} \phi(p) - \nabla_{q} \langle q - p, \nabla\phi (p) \rangle\\
	& = \nabla \phi(q) - \nabla_{q} \langle q - p, \nabla\phi (p) \rangle\\
	& = \nabla \phi(q) - \nabla_{q} \langle q, \nabla\phi (p) \rangle\\
	& = \nabla \phi(q) - \nabla_{q} ( q_1\frac{\partial}{\partial p_1}\phi (p) + \ldots + q_k\frac{\partial}{\partial p_k}\phi (p))\\
	& = \nabla \phi(q) - (\frac{\partial}{\partial p_1}\phi (p), \ldots,\frac{\partial}{\partial p_k}\phi (p))\\
	& = \nabla \phi(q) - \nabla \phi(p)
\end{split}
\end{equation*}
\end{proof}

\begin{lemma} \label{lem:two}
Let $d_\phi : ri(\Omega) \times \Omega \to \mathbb{R}$ be a Bregman divergence. 
Then for any $p, q \in ri(\Omega)$ and $z \in \Omega$ the following holds:
$$d_\phi (p, z) - d_\phi (q, z) = \langle z-q, \nabla_q d_\phi (p,q) \rangle + d_\phi (p,q).$$
\end{lemma}
\begin{proof} Simplifying from the definition of Bregman divergence gives:
\begin{equation*}
\begin{split}
d_\phi (p,z) - d_\phi (q, z) & = (\phi(z) - \phi(p) - \langle z-p, \nabla\phi(p) \rangle) - (\phi(z) - \phi(q) - \langle z-q, \nabla\phi(q) \rangle) \\
    & = \phi(q) -  \phi(p) + \langle z-q, \nabla\phi(q) \rangle  - \langle z-p, \nabla\phi(p) \rangle
\end{split}
\end{equation*}
Using Lemma~\ref{lem:one} to rewrite the third term yields:
\begin{equation*}
\begin{split}
    & = \phi(q) -  \phi(p) + \langle z-q,  \nabla_q d_\phi(p,q) + \nabla\phi(p) \rangle  - \langle z-p, \nabla\phi(p)\rangle \\
    & = \phi(q) -  \phi(p) + \langle z-q,  \nabla_q d_\phi(p,q) \rangle+ \langle z-q, \nabla\phi(p) \rangle  - \langle z-p, \nabla\phi(p)\rangle \\
    & = \phi(q) -  \phi(p) + \langle z-q,  \nabla_q d_\phi(p,q) \rangle - \langle q, \nabla\phi(p) \rangle  + \langle p, \nabla\phi(p)\rangle \\
    & = \phi(q) -  \phi(p) + \langle z-q,  \nabla_q d_\phi(p,q) \rangle - \langle q - p, \nabla\phi(p) \rangle \\
    & = \langle z-q,  \nabla_q d_\phi(p,q) \rangle + \phi(q) -  \phi(p) - \langle q - p, \nabla\phi(p) \rangle \\
    & = \langle z-q, \nabla_q d_\phi(p,q) \rangle  + d_\phi (p,q)
\end{split}
\end{equation*}
\end{proof}

\begin{lemma} \label{lem:three}
Let $d_\phi$ be a Bregman divergence, let $p$ be a set of predictions, and $\pi$ be a class distribution over $k$ classes. Denoting $a^\star= \alpha^\star(p, \pi)$, the following holds for any $q\in Q_\pi$:
$$\frac{1}{n}\sum_{i=1}^n 
(d_\phi(p_{i\cdot},q_{i\cdot}) 
- d_\phi(a^\star_{i\cdot}, q_{i\cdot})) 
= \frac{1}{n}\sum_{i=1}^n (d_\phi(p_{i\cdot}, a_{i\cdot}^\star))$$
%\emph{This says that our adjustment function guarantees an expected reduction of divergence and gives a decomposition of loss.} 
\end{lemma}
\begin{proof} 
In the following we will use a simplified notation and write 
$p_i$, $q_i$, $a^\star_i$ instead of 
$p_{i\cdot}$, $q_{i\cdot}$, $a^\star_{i\cdot}$.
Using Lemma~\ref{lem:two} we can write
$$\frac{1}{n}\sum_{i=1}^n (d_\phi(p_i,q_i) - d_\phi(a^\star_i, q_i)) = \frac{1}{n}\sum_{i=1}^n (\langle q_i - a^\star_i, \nabla_{a^\star_i} d_\phi (p_i, a^\star_i) \rangle + d_\phi(p_i, a_i^\star))$$
If we can prove that
$$\sum_{i=1}^n \langle q_i - a^\star_i, \nabla_{a^\star_i} d_\phi (p_i, a^\star_i) \rangle = 0$$
then the proof will be complete. So we begin by using the method of Lagrange multipliers to define what each $\nabla_{a^\star_i} d_\phi (p_i, a^\star_i)$ is for each $i \in \{1,\dots,n\}$. We rewrite the original argument minimization problem. Keep note our new function will have $n \times k$ variables from $a$, and $n$ variables from our first constraint, and $k$ variables from our second constraint.
$$F(a, \theta, \lambda) = \sum_{i=1}^n d_\phi(p_i, a_i) + \sum_{i=1}^n \theta_i (1 - \sum_{j=1}^k a_{i,j}) + \sum_{j=1}^k \lambda_j (\pi_j - \frac{1}{n}\sum_{i=1}^n a_{i,j})$$
Minimum is when
$$\nabla F(a, \theta, \lambda) = \textbf0.$$
Let's expand the gradient.
$$\nabla F(a, \theta, \lambda) = (\nabla_a F(a, \theta, \lambda), \nabla_\theta F(a, \theta, \lambda), \nabla_\lambda F(a, \theta, \lambda))$$
Let's expand the first term. For simplicity's sake we will represent $\nabla_a$ as a matrix, but it is a vector in actuality.
\[
\nabla_a F(a, \theta, \lambda)
=
\begin{bmatrix}
    \frac{\partial}{\partial a_{1,1}} F(a, \theta, \lambda)        & \dots & \frac{\partial}{\partial a_{1,k}} F(a, \theta, \lambda) \\
    \hdotsfor{3} \\
    \frac{\partial}{\partial a_{n,1}} F(a, \theta, \lambda)       & \dots & \frac{\partial}{\partial a_{n,k}} F(a, \theta, \lambda)
\end{bmatrix}
\]
\[
=
\begin{bmatrix}
    \theta_1 + \lambda_1 + \frac{\partial}{\partial a_{1,1}} d_\phi (p_1, a_1)       & \dots & \theta_1 + \lambda_k + \frac{\partial}{\partial a_{1,k}} d_\phi (p_1, a_1) \\
    \hdotsfor{3} \\
    \theta_n + \lambda_1 + \frac{\partial}{\partial a_{n,1}} d_\phi (p_n, a_n)      & \dots & \theta_n + \lambda_k + \frac{\partial}{\partial a_{n,k}} d_\phi (p_n, a_n)
\end{bmatrix}
\]
We can now see that for each entry $(i,j)$ in $\nabla_a F(a, \theta, \lambda)$ to equal $0$, then each 
$$\frac{\partial}{\partial a_{i,j}} d_\phi (p_i, a_i) = - \theta_i - \lambda_j$$
Since $a^\star$ is at the minimum, this implies
$$\nabla_{a^\star_i} d_\phi (p_i, a^\star_i) = ( - \theta_i - \lambda_1,..., - \theta_i - \lambda_k)$$
We can now write out
\begin{equation*}
\begin{split}
\sum_{i=1}^n \langle q_i - a^\star_i, \nabla_{a^\star_i} d_\phi (p_i, a^\star_i) \rangle &= \sum_{i=1}^n \sum_{j=1}^k (q_{i,j} - a^\star_{i,j})(- \theta_i - \lambda_j) \\
	&= \sum_{i=1}^n \sum_{j=1}^k (q_{i,j} - a^\star_{i,j})(- \theta_i) + \sum_{i=1}^n \sum_{j=1}^k (q_{i,j} - a^\star_{i,j})(- \lambda_j)\\
	&= \sum_{i=1}^n (- \theta_i)\sum_{j=1}^k (q_{i,j} - a^\star_{i,j}) + \sum_{j=1}^k (- \lambda_j)\sum_{i=1}^n (q_{i,j} - a^\star_{i,j})\\
\end{split}
\end{equation*}
We know from the constraints that each row and column of $q-a^\star$ sums to 0.
$$\sum_{j=1}^k (q_{i,j} - a^\star_{i,j})=0 \text{~and~} \sum_{i=1}^n (q_{i,j} - a^\star_{i,j})=0$$
So it's clear that
$$\sum_{i=1}^n (- \theta_i)\sum_{j=1}^k (q_{i,j} - a^\star_{i,j}) + \sum_{j=1}^k (- \lambda_j)\sum_{i=1}^n (q_{i,j} - a^\star_{i,j}) = 0$$
\end{proof}

\begin{theorem}
Let $\alpha^\star$ be the unbounded general adjuster corresponding to the Bregman divergence $d_\phi$.
Then $\alpha^\star$ is coherent with $d_\phi$.
\end{theorem}
\begin{proof} Let $a^\star= \alpha^\star(p, \pi)$. For $\alpha^\star$ to be coherent. the following equation must be satisfied following the definition of coherence (we use notation $e_i$ instead of $c_i$ to emphasise that these are unit vectors, we use letters $i$ and $j$ instead of $j$ and $j'$, and letter $x$ to stand for a row in matrices $a^\star$ and $p$):
$$d_\phi(a^\star_x, e_i) - d_\phi(a^\star_x,e_j) - d_\phi(p_x, e_i) + d_\phi(p_x,e_j) \stackrel{?}{=} const_{i,j}$$
We can just use the definition of divergence and properties of vectors to get the equation into a new form.
\begin{equation*}
\begin{split}
const_{i,j} &\stackrel{?}{=} d_\phi(a^\star_x, e_i) - d_\phi(a^\star_x,e_j) - d_\phi(p_x, e_i) + d_\phi(p_x,e_j) \\
	&=  \phi(e_i) - \phi(a^\star_x) - \langle e_i - a^\star_x, \nabla \phi(a^\star_x)\rangle \\
	&\text{~~~~}- \phi(e_j) + \phi(a^\star_x) + \langle e_j - a^\star_x, \nabla \phi(a^\star_x)\rangle \\ 
	&\text{~~~~}-\phi(e_i) + \phi(p_x) + \langle e_i - p_x, \nabla \phi(p_x)\rangle \\
	&\text{~~~~}+  \phi(e_j) - \phi(p_x) - \langle e_j - p_x, \nabla \phi(p_x)\rangle\\
	&=  \langle e_j - a^\star_x, \nabla \phi(a^\star_x)\rangle \\
	&\text{~~~~}-  \langle e_i - a^\star_x, \nabla \phi(a^\star_x)\rangle \\
	&\text{~~~~}-  \langle e_j - p_x, \nabla \phi(p_x)\rangle \\
	&\text{~~~~}+ \langle e_i - p_x,  \nabla \phi(p_x)\rangle\\ %\nabla \phi(p)\rangle\\
&=  \langle e_j - e_i, \nabla \phi(a^\star_x)\rangle - \langle e_j - e_i,  \nabla \phi(p_x)\rangle\\ 
&=  \langle e_j - e_i, \nabla \phi(a^\star_x) -  \nabla \phi(p_x)\rangle
\end{split}
\end{equation*}
From our earlier theorem, we know.
$$\langle e_j - e_i, \nabla \phi(a^\star_x) -  \nabla \phi(p_x)\rangle =  \langle e_j - e_i, \nabla_{a^\star_x}d_\phi(p_x,a^\star_x)\rangle$$
We know from the proof in Lemma~\ref{lem:three} that $\nabla_{a^\star_x}d_\phi(p_x,a^\star_x)$ is defined by the sum of two variables that depend on $i$ and $j$, $\theta$ and $\lambda$.
That means $const_{i,j}  = \langle e_j - e_i, \nabla^\star_{a^\star_x}d_\phi(p_x,a^\star_x)\rangle$ only depends on $i$ and $j$ and not on $x$, matching the definition of coherence.
\end{proof}

\begin{theorem}
Let $d_\phi$ be a Bregman divergence, let $p$ be a set of predictions, and $\pi$ be a class distribution over $k$ classes. Suppose $a\in Q_\pi$ is such that for any $y\in Q_\pi$ the decomposition of Eq.(\ref{eq:decomp}) holds.
%\begin{align*}
%\frac{1}{n}\sum_{i=1}^n d_\phi(p_{i\cdot},y_{i\cdot}) =
%\frac{1}{n}\sum_{i=1}^n d_\phi(p_{i\cdot},a_{i\cdot}) +
%\frac{1}{n}\sum_{i=1}^n d_\phi(a_{i\cdot},y_{i\cdot})
%\end{align*}
Then $a=\alpha^\star(p,\pi)$.
\end{theorem}
\begin{proof}
We prove by contradiction and assume that $a \neq \alpha^\star(p, \pi)$.
Take the case where $q = \alpha^\star(p, \pi)$.
We can rewrite the theorem's equality to
$$\sum_{i=1}^n d_\phi(p_i,q_i) = \sum_{i=1}^n (d_\phi(p_i, a_i) + d_\phi(a_i, q_i)).$$
By the definition of $\alpha^\star$ we have
$\sum_{i=1}^n d_\phi(p_i,q_i) < \sum_{i=1}^n d_\phi(p_i, a_i))$
and by the definition of Bregman divergence 
$\sum_{i=1}^n d_\phi(a_i, q_i) > 0$.
Therefore, 
$$\sum_{i=1}^n d_\phi(p_i,q_i) < \sum_{i=1}^n (d_\phi(p_i, a_i) + d_\phi(a_i, q_i))$$.
We have a contradiction, so the assumption was false.
\end{proof}

\begin{lemma} \label{lem:four}
Let $d_\phi$ be a Bregman divergence, let $p$ be a set of predictions, and $\pi$ be a class distribution over $k$ classes. Denoting $a^\square = \alpha^\square(p, \pi)$, the following holds for any $q\in Q^\square_\pi$:
$$\sum_{i=1}^n \langle q_i-a^\square_i, \nabla_{a^\square_i} d_\phi (p_i,a^\square_i) \rangle \geq 0$$ 
\end{lemma}
\begin{proof}
This is pretty much like the proof of Lemma~\ref{lem:three} except we use the Karush-Kuhn-Tucker method to add our extra set of inequality constraints.
\begin{align*}
F(a, \theta, \lambda, \psi) 
&= \sum_{i=1}^n d_\phi(p_i, a_i) + \sum_{i=1}^n \theta_i (1 - \sum_{j=1}^k a_{i,j}) \\
&+ \sum_{j=1}^k \lambda_j (\pi - \frac{1}{n}\sum_{i=1}^n a_{i,j}) + \sum_{i=1}^n \sum_{j=1}^k \psi_{i,j} (-a_{i,j})
\end{align*}
Minimum is when
$$\nabla F(a, \theta, \lambda, \psi) = \textbf0.$$
Let's expand the gradient.
$$\nabla F(a, \theta, \lambda, \psi) = (\nabla_a F(a, \theta, \lambda), \nabla_\theta F(a, \theta, \lambda), \nabla_\lambda F(a, \theta, \lambda), \nabla_\psi F(a, \theta, \lambda))$$
Let's expand the first term. For simplicity's sake we will represent $\nabla_a$ as a matrix, but it is a vector in actuality.
\[
\nabla_a F(a, \theta, \lambda)
=
\begin{bmatrix}
    \frac{\partial}{\partial a_{1,1}} F(a, \theta, \lambda)        & \dots & \frac{\partial}{\partial a_{1,k}} F(a, \theta, \lambda) \\
    \hdotsfor{3} \\
    \frac{\partial}{\partial a_{n,1}} F(a, \theta, \lambda)       & \dots & \frac{\partial}{\partial a_{n,k}} F(a, \theta, \lambda)
\end{bmatrix}
\]
\[
=
\begin{bmatrix}
    \theta_1 + \lambda_1 - \psi_{1,1} + \frac{\partial}{\partial a_{1,1}} d_\phi (p_1, a_1)       & \dots & \theta_1 + \lambda_k - \psi_{1,k} + \frac{\partial}{\partial a_{1,k}} d_\phi (p_1, a_1) \\
    \hdotsfor{3} \\
    \theta_n + \lambda_1 - \psi_{n,1} + \frac{\partial}{\partial a_{n,1}} d_\phi (p_n, a_n)      & \dots & \theta_n + \lambda_k - \psi_{n,k} + \frac{\partial}{\partial a_{n,k}} d_\phi (p_n, a_n)
\end{bmatrix}
\]
We can now see that for each entry $(i,j)$ in $\nabla_a F(a, \theta, \lambda, \psi)$ to equal $0$, then each 
$$\frac{\partial}{\partial a_{i,j}} d_\phi (p_i, a_i) = \psi_{i,j} - \theta_i - \lambda_j$$
Since $a^\square$ is at the minimum, this implies
$$\nabla_{a^\square_i} d_\phi (p_i, a^\square_i) = ( \psi_{i,1} - \theta_i - \lambda_1,..., \psi_{i,k} - \theta_i - \lambda_k)$$
We can write out
\begin{equation*}
\begin{split}
\sum_{i=1}^n \langle q_i - a^\square_i, \nabla_{a^\square_i} d_\phi (p_i, a^\square_i) \rangle &= \sum_{i=1}^n \sum_{j=1}^k (q_{i,j} - a^\square_{i,j})(\psi_{i,j} - \theta_i - \lambda_j) \\
	&=  \sum_{i=1}^n \sum_{j=1}^k \psi_{i,j}(q_{i,j} - a^\square_{i,j}) \\
	&\text{~~~~} + \sum_{i=1}^n \sum_{j=1}^k (q_{i,j} - a^\square_{i,j})(- \theta_i) \\
	&\text{~~~~} + \sum_{i=1}^n \sum_{j=1}^k (q_{i,j} - a^\square_{i,j})(- \lambda_j)\\
\end{split}
\end{equation*}
We know from the earlier proof of Lemma~\ref{lem:three} that the last two terms equal $0$, which leaves us
$$\sum_{i=1}^n \langle q_i - a^\square_i, \nabla_{a^\square_i} d_\phi (p_i, a^\square_i) \rangle = \sum_{i=1}^n \sum_{j=1}^k \psi_{i,j}(q_{i,j} - a^\square_{i,j})$$
Now let's look at what each $\psi_{i,j}$ actually is. The KKT conditions require that each $\psi_{i,j} \geq 0$ and that $\psi_{i,j}(- a^\square_{i,j}) = 0$. This implies that the only times that $\psi_{i,j} \neq 0$ is when $a^\square_{i,j} = 0$ in which case $\psi_{i,j} \geq 0$. \\
In our double sum, we only have to be concerned with the terms that have an $a^\square_{i,j} = 0$ (all the other terms will be $0$, since if $a^\square_{i,j} \neq 0$ then $\psi_{i,j} = 0$). In these cases, $q_{i,j} - a^\square_{i,j} > 0$ since $q_{i,j} \geq 0$ by the constraint. $q_{i,j} - a^\square_{i,j} \geq 0$ and $\psi_{i,j} \geq 0$, so $\sum_{i=1}^n \sum_{j=1}^k \psi_{i,j}(q_{i,j} - a^\square_{i,j}) \geq 0$.
\end{proof} 

\begin{theorem}
Let $d_\phi$ be a Bregman divergence, let $p$ be a set of predictions, and $\pi$ be a class distribution over $k$ classes. Then for any $y\in Q^\square_\pi$ the following holds:
\begin{align*}
&\sum_{i=1}^n (d_\phi(p_{i\cdot},y_{i\cdot}) - d_\phi(a_{i\cdot}^\square, y_{i\cdot})) \geq \\
&\geq\sum_{i=1}^n d_\phi(p_{i\cdot}, a_{i\cdot}^\square) \geq 
\sum_{i=1}^n d_\phi(p_{i\cdot}, a_{i\cdot}^\star)=
\sum_{i=1}^n (d_\phi(p_{i\cdot},y_{i\cdot}) - d_\phi(a_{i\cdot}^\star, y_{i\cdot}))
\end{align*}
%\frac{1}{n}\sum_{i=1}^n d_\phi(p_{i\cdot},y_{i\cdot}) =
\end{theorem}
\begin{proof} Writing out the difference and using the previous Lemma~\ref{lem:four} with $q=y$ gives:
\begin{equation*}
\begin{split}
\sum_{i=1}^n (d_\phi(p_{i\cdot}, y_{i\cdot}) - d_\phi(a_{i\cdot}^\square,y_{i\cdot})) &= \sum_{i=1}^n (\langle a_{i\cdot}^\square-y_{i\cdot}, \nabla_a d_\phi (p_{i\cdot},a_{i\cdot}^\square) \rangle + d_\phi (p_{i\cdot},a_{i\cdot}^\square)) \\
	&=  \sum_{i=1}^n \langle a_{i\cdot}^\square-y_{i\cdot}, \nabla_a d_\phi (p_{i\cdot},a_{i\cdot}^\square) \rangle + \sum_{i=1}^n d_\phi (p_{i\cdot},a_{i\cdot}^\square) \\
	&\geq \sum_{i=1}^n d_\phi (p_{i\cdot},a_{i\cdot}^\square) 
\end{split}
\end{equation*}
We know $\sum_{i=1}^n d_\phi (p_{i\cdot},a_{i\cdot}^\square) \geq \sum_{i=1}^n d_\phi (p_{i\cdot},a^\star_i)$ since $a^\star$ and $a^\square$ are either equal or $a^\square$ would have been chosen over $a^\star$ in $\alpha^\star$'s minimization task.
The rest follows from Lemma~\ref{lem:four}.
\end{proof}